\RequirePackage{fix-cm}
\documentclass[twocolumn]{svjour3}  
\smartqed 
\usepackage{graphicx}
\usepackage{tabularx}
\usepackage{bm}
\usepackage{amsmath}
\usepackage{amsfonts}
\usepackage{mathtools}
\usepackage{overpic}
\usepackage{hhline}
\usepackage{color}
\usepackage{booktabs}
\usepackage{url}

\renewcommand{\div}[1]{\operatorname{div} \left ( #1 \right )}
\newcommand{\tr}[1]{\mathrm{tr} \left ( #1 \right )}

\newcommand{\pd}[2]{\frac{\partial #1}{\partial #2}}
\newcommand{\eg}[0]{\emph{e.g.}}
\newcommand{\ie}[0]{\emph{i.e.}}
\newcommand{\dive}{\,\mathrm{div}}
\newcommand{\figwidth}{0.24}

\journalname{Journal of Mathematical Imaging and Vision}

\def\imagetop#1{\vtop{\null\hbox{#1}}}

\begin{document}
\title{Mapping-based Image Diffusion}

\author{Freddie {\AA}str{\"o}m \and Michael Felsberg \and George Baravdish}

\institute{F. {\AA}str{\"o}m, \at
  Heidelberg University \\
  Heidelberg Collaboratory for Image Processing (HCI) \\
  Im Neuenheimer Feld 205 \\
  69 120  Heidelberg, Germany \\
  Tel.: +49 (0) 662 1541 4841 \\
  \email{freddie.astroem@iwr.uni-heidelberg.de}  
  \and
  M. Felsberg \at
  Link{\"o}ping University \\
  Department of Electrical Engineering  \\
  581 83 Link{\"o}ping, Sweden   \\
  Tel.: +46 (0) 13 28 2460\\
  \email{michael.felsberg@liu.se}  
  \and
  G. Baravdish \at
  Link{\"o}ping University \\
  Department of Science and Technology \\
  581 83 Link{\"o}ping, Sweden   \\
  Tel.: +46 (0) 11 36 3118 	\\
  \email{george.baravdish@liu.se}  
}

\date{Received: date / Accepted: date}

\maketitle

\begin{abstract}
 In this work, we introduce a novel tensor-based functional for targeted image enhancement and denoising. Via explicit regularization, our formulation incorporates application dependent and contextual information using first principles. Few works in literature treat variational models that describe both application dependent information and contextual knowledge of the denoising problem. We prove the existence of a minimizer and present results on tensor symmetry constraints, convexity, and geometric interpretation of the proposed functional. We show that our framework excels in applications where nonlinear functions are present such as in gamma correction and targeted value range filtering. We also study general denoising performance where we show comparable results to dedicated PDE-based state of the art methods. 

\end{abstract}

\begin{figure*}[tbh]
   \centering
	\includegraphics[width=1\textwidth]{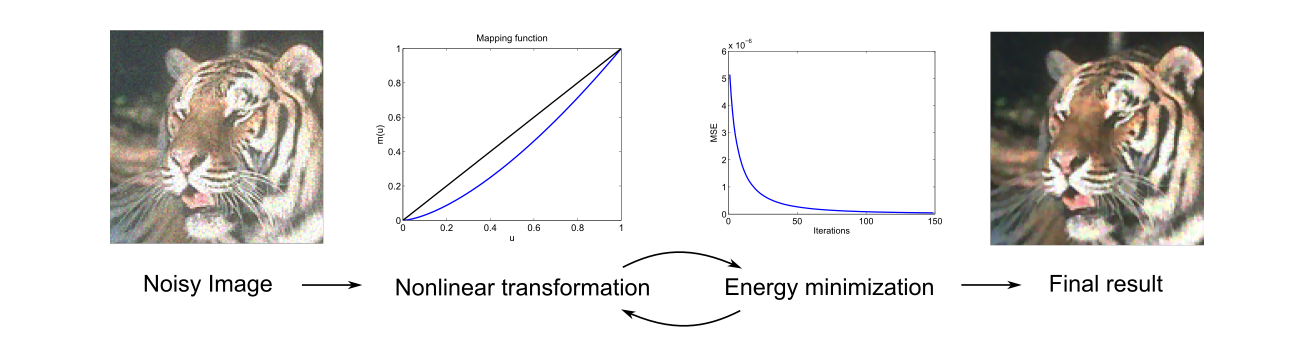}\\
	\caption{The presented framework studies image denoising from the viewpoint that the image data is not only corrupted by ``noise" but also pertubated by some nonlinear function. In this figure, the nonlinear function corrects for the overexposed image. We incorporate this nonlinear function called a ``mapping'' as a feedback component in the iterative update scheme. For additional examples of the mapping function see Figure \ref{fig:example_synthetic_butterfly}.}
	\label{fig:overview}
\end{figure*}

\section{Introduction}
\label{intro}
A simple approach to noise reduction is to apply an isotropic convolution kernel to the image data, \eg, normalized Gaussian filters or box filters~\cite{Koenderink1984}. Due to the simplicity of the approach, spatial and contextual information are neglected and therefore the output image is blurred. A common approach to adaptive filtering is to let the image gradient control the amount of filtering see, e.g., the seminal work of Perona and Malik (P-M) \cite{Perona90scale-spaceand}. They introduced an edge-stopping function to limit the filtering close to edges and corners. The P-M filter is often referred to as a nonlinear isotropic diffusion filter \cite{Weickert96anisotropicdiffusion}. The filter's edge-stopping function acts to preserve image structure, thereby obtaining superior image quality compared to isotropic filtering. The successor to P-M filtering was introduced by Weickert \cite{Weickert96anisotropicdiffusion} where additional adaptivity to image structure is achieved by using a second-order tensor. The main idea of the tensor-based formulation is to smooth parallel to image structures. The concept is based on the structure tensor \cite{diva2:274026,Forstner87}. The tensor is a windowed second moment matrix that describes the local orientation as a tensor field. Transforming the tensor field to align parallel to the image structure, one obtains a structure-preserving filtering scheme. The transformed tensor field is then used in the filtering scheme and is denoted as the diffusion tensor.

\renewcommand{\figwidth}{0.32}
\begin{figure*}[t]
  \newcommand{\tabwidth}{1}	
  \centering
  \hfill  \includegraphics[width=\figwidth\textwidth]{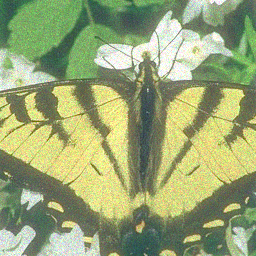} \hfill \hfill
  \hfill  \includegraphics[width=\figwidth\textwidth]{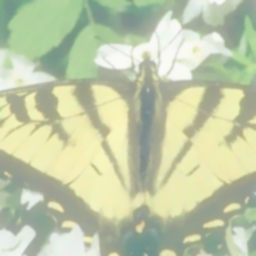} \hfill \hfill
  \hfill  \includegraphics[width=\figwidth\textwidth]{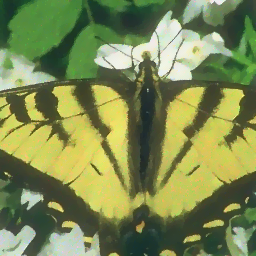} \hfill \phantom{I} \\
  \begin{tabularx}{\tabwidth\textwidth}{XXX}
    \textbf{Gamma correction} &  \textbf{Mapping} &  \textbf{Result} \\
  \end{tabularx}\\
  \hfill  \includegraphics[width=\figwidth\textwidth]{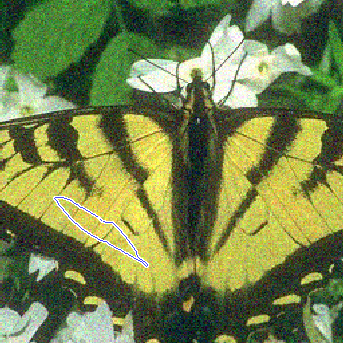} \hfill \hfill
  \hfill  \includegraphics[width=\figwidth\textwidth]{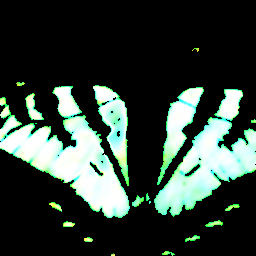} \hfill \hfill
  \hfill  \includegraphics[width=\figwidth\textwidth]{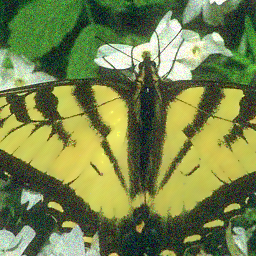} \hfill \phantom{I} \\
  \begin{tabularx}{\tabwidth\textwidth}{XXX}
     \textbf{Targeted filtering} & \textbf{Mapping} &  \textbf{Result} \\
  \end{tabularx}\\
  \hfill  \includegraphics[width=\figwidth\textwidth]{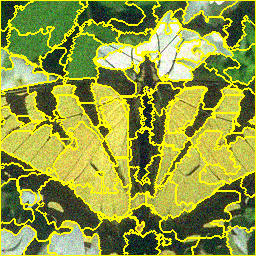} \hfill \hfill
  \hfill  \includegraphics[width=\figwidth\textwidth]{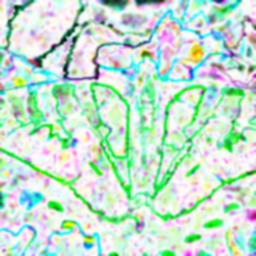} \hfill \hfill
  \hfill  \includegraphics[width=\figwidth\textwidth]{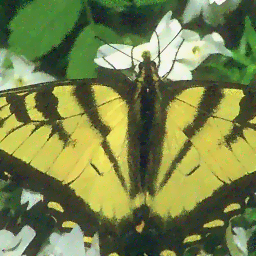} \hfill \phantom{I} \\
  \begin{tabularx}{\tabwidth\textwidth}{XXX}
    \textbf{Oversegmentation} &  \textbf{Mapping} &  \textbf{Result} \\
  \end{tabularx}\\
  \caption{In this work we introduce a mapping function in the energy \eqref{eq:main_R}. We consider three applications, each with different definitions of the mapping. First we perform gamma correction, illustrated on the first row. In the second application (second row) we perform targeted filtering. In the third application we perform general denoising by estimating the mapping from the image data via an oversegmentation map.}
  \label{fig:example_synthetic_butterfly}
\end{figure*}

The linear diffusion scheme is the solution of the diffusion equation, a partial differential equation (PDE), and it is closely related to the notion of scale-space \cite{lindeberg1993scale,Koenderink1984}. Therefore it has been of interest to investigate also P-M formulation and its successors of adaptive image filtering in terms of a variational framework. The motivation is that the variational framework allows for a generic approach to model the denoising problem and related image processing applications. By modifying the objective function this enables the inclusion of existing applications, \eg, active contours \cite{Kass88snakes:active}, deblurring \cite{661187,Oliveira20091683}, optical flow \cite{HornShunk1981,1420392,4409000}, and inpainting \cite{935036,Chan02euler'selastica}. Denote by $u:\mathbb{R}^2 \to \mathbb{R}$ the image data. The variational approach to image diffusion is to model an energy functional, $E$, of the form
\begin{equation}\label{eq:mainEFR}
	E(u) = F(u) + \lambda R(u),
\end{equation}
where $F$ is a fidelity term and the positive constant $\lambda$ determines the influence of $R$, the regularization term. The aim is to find the minimizer of the energy functional $E$. Here we are interested in tensor-based formulations of the regularization term. In particular we study and introduce an application dependent, tensor-driven anisotropic diffusion scheme. Furthermore, we introduce the gradient energy tensor~\cite{diva2:269100} as an alternative to the structure tensor \cite{Forstner87}.

\subsection{Approach}
Let $\Omega$ be a two-dimensional domain, denoting the support of the image $u$ in pixels. Assume $u$ to be differentiable and denote the gradient operator as $\nabla = ( \partial_{x}, \partial_{y} )^{\top}$, where $^{\top}$ is the transpose. The gradient of $u$ is defined as $\nabla u = ( \partial_{x} u, \partial_{y}u )^{\top}$. Throughout this work we introduce a \emph{mapping function}. We let this mapping be a differentiable nonlinear function $m : \mathbb{R}\to\mathbb{R}$, which describes a mapping of the \emph{image value range}, $u$. For notational convenience we write $\nabla m(u)$ rather than $\nabla(m(u(x,y)))$ where $(x,y) \in \Omega$ to express the differentiation of the mapping w.r.t. its argument. Given a noisy image, the mapping function can be estimated or defined from prior knowledge of the enhancement problem. Figure~\ref{fig:overview} illustrates the application of gamma correction. By letting the mapping function be an (estimated) prior of the nonlinear transform we obtain a feedback component in the optimization scheme tailored for the specific application. This mapping is not only  capable of preserving spatial features (\eg, corners, lines) but also preserves (or smooth) image data in certain image intensity ranges, \ie, the filter also depends on the image value range. 

The mapping can be either \emph{global} or \emph{local} depending on the application. We consider a \emph{global} gamma mapping (Section \ref{sec:gc}), \emph{local} targeted region filtering (Section \ref{sec:tf}), and global region filtering via \emph{local} oversegmentation maps (Section \ref{sec:colour_eval}). A preview of these applications is given in Figure \ref{fig:example_synthetic_butterfly} and Table \ref{table:overview_app} shows the corresponding interpretations of the mapping function. We emphasize that the determination of a relevant mapping function is highly application dependent. For readers familiar with anisotropic diffusion \cite{Weickert96anisotropicdiffusion,Perona90scale-spaceand} we remark that the selection of mapping function can be thought as the accurate selection of a diffusivity function. However, while the diffusivity function is gradient and data dependent, the mapping function is intensity and application dependent. In connection to the variational model \eqref{eq:mainEFR} we study regularization terms of the form
\begin{equation}\label{eq:main_R}
  R(u) = \int_\Omega \nabla m(u)^{\top} W(\nabla m(u)) \nabla m(u) \; d{\bm{x}},
\end{equation}
where ${\bm{x}} = (x,y)$ and $W \in \mathbb{R}^{2 \times 2}$ is a positive semi-definite tensor that incorporates orientation dependent information.

\renewcommand{\figwidth}{0.32}
\begin{figure*}[t]
  \newcommand{\tabwidth}{1}	
  \centering
  \hfill  \includegraphics[width=\figwidth\textwidth]{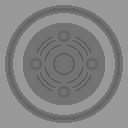} \hfill \hfill
  \hfill  \includegraphics[width=\figwidth\textwidth]{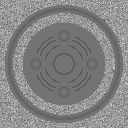} \hfill \hfill
  \hfill  \includegraphics[width=\figwidth\textwidth]{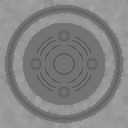} \hfill \phantom{I} \\
  \begin{tabularx}{\tabwidth\textwidth}{XXX}
     \textbf{Original} &  \textbf{Noisy} &  \textbf{Total variation mapping (see \eqref{eq:TVm})} \\
  \end{tabularx}\\
  \hfill  \includegraphics[width=\figwidth\textwidth]{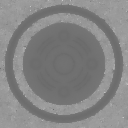} \hfill \hfill
  \hfill  \includegraphics[width=\figwidth\textwidth]{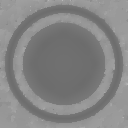} \hfill \hfill
  \hfill  \includegraphics[width=\figwidth\textwidth]{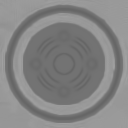} \hfill \phantom{I} \\
  \begin{tabularx}{\tabwidth\textwidth}{XXX}
    \textbf{Markov random field}  & \textbf{Total generalized variation}  &  \textbf{Block-matching} \\
    MRF \cite{5539844} & TGV \cite{Bredies:2010:TGV:1958729.1958739}  & BM3D \cite{DBLP:journals/tip/DabovFKE07} 
  \end{tabularx}\\
  \caption{Synthetic example where an image has been corrupted with additive Gaussian noise with standard deviation 20 in a specific value range. By appropriately defining the mapping in \eqref{eq:main_R} we are able to preserve image details oversmoothed by MRF, TGV and BM3D. In this example the mapping function was of sigmoidal type, \eqref{eq:sigmoidal} with $a=100$ and $c=1/2$.}
  \label{fig:example_synthetic}
\end{figure*}

\subsection{Synthetic example}
Figure \ref{fig:example_synthetic} illustrates the advantage when performing noise reduction with targeted filtering. In this image, the noise component is only present in the background of the image. The foreground, or object, consisting of circles of different diameters and intensities are not corrupted by noise. Based on this prior knowledge we set, in this example, the mapping as to be sigmoidal, \ie
\begin{equation}\label{eq:sigmoidal}
  m'(u;a,c) = (1 + \exp(-a(u-c)))^{-1}, \quad a,c > 0,
\end{equation}
The parameters were set to $a=100$ and $c=1/2$ as $u\in [0,1]$, these parameters correspond to a smooth approximation of a step-function translated to $1/2$. As an instance of our framework, we use total variation with a mapping function (introduced in Section \ref{sec:unifying_functional}). Visually, comparing the result to Markov Random Field~(MRF) \cite{5539844}, Total Generalized Variation (TGV) \cite{Bredies:2010:TGV:1958729.1958739} and BM3D \cite{DBLP:journals/tip/DabovFKE07} our approach appears more similar to the noise free image. MRF, TGV and BM3D all oversmooth the four small circles and details. We note however that in this simple example, assuming that the input image can be accurately partitioned in noise-free and noisy regions, one can naturally apply a restoration method on these partitions. However, in the case of more complex image regions, see Figure \ref{fig:example_synthetic_butterfly} second and third row, such heuristic partitioning of the image value range is in general a highly non-trivial task. Furthermore, by incorporating the mapping into the filtering process, one does not require an additional heuristic post-processing step to accurately fuse these independent regions. 

\subsection{Overview and summary of contributions}
The main contribution of this work is to study the effects of the mapping function in the tensor-based regularization term \eqref{eq:main_R}. The corresponding minimizer is presented in Theorem~\ref{thm:R}. We show in the experimental part that the inclusion of the nonlinear mapping, as a component of the filtering scheme, improves the filtering performance compared to applying the mapping as a separate pre- or post-processing step. The denoising applications we consider are gamma correction of the intensity range, targeted filtering, and adaptive denoising where the image value distributions are used to drive the filtering process. Furthermore, we introduce the gradient energy tensor (GET) \cite{diva2:269100} as an orientation dependent term into the regularization term. Our formulation allows us to utilize both the \emph{eigenvalues} and \emph{eigenvectors} of the tensor. 

In Section \ref{sec:related_works} we give an overview of denoising approaches related to the current work. Section \ref{sec:review_diffusion} further details the diffusion framework and introduces some additional notation. In Section \ref{sec:unifying_functional} we present the main results of this paper, summarized in Theorem \ref{thm:R}. To minimize the proposed functional we prove the existence of its Euler-Lagrange (E-L) equation. Furthermore, we derive properties in the subsequent analysis useful for image denoising applications. We show that the proposed functional describes several variational denoising methods as special cases and we extend these to include the mapping function. In particular we extend our gradient energy total variation formulation \cite{Astrom2014} to also include the mapping function. In the same section we also give some additional properties of the energy and its E-L equation. Section \ref{sec:gradient_energy_tensor} briefly reviews the GET tensor and presents an empirical evaluation where the tensor is compared to the structure tensor for orientation estimation. The GET tensor is used to formulate the gradient energy total variation scheme with a mapping function in Section \ref{sec:GETVm}. 

In the experimental evaluation, Section \ref{sec:gc}, we show empirically that providing the nonlinear mapping of the image data as a feedback component yields an improved result compared to (I) transforming the image value range and then denoise, or (II) denoise and then transform the image value range. In Section \ref{sec:tf} we let the mapping define a user-set value range for targeted filtering. Section \ref{sec:colour_eval}, we evaluate our framework on a third application where we let the filtering process be based on an oversegmentation map. Section \ref{sec:discussion} gives an extensive discussion on the difference between local and global mapping functions and their applicability in the context of image denoising. A preview of the above mentioned applications and their relation to the different classes of mappings are given in Table \ref{table:overview_app}. Recall that the mappings are nonlinear. Section \ref{sec:conclusion} concludes this work. 

\begin{table}[t]
  \renewcommand*{\arraystretch}{1.3}
  \newcommand{\parwidth}{1\linewidth}
  \begin{tabularx}{1\linewidth}{l}
    \textbf{Applications} \\
    \hline
    \textbf{I}. Gamma correction  \\
    \hspace{5mm} \parbox{\parwidth}{\emph{Global mapping applied pixel-wise}} \\
    \textbf{II}. Targeted filtering \\
    \hspace{5mm} \parbox{\parwidth}{\emph{Local mapping applied globally (on the value range)}} \\
    \textbf{III}. Oversegmentation \\
    \hspace{5mm} \parbox{\parwidth}{\emph{Local mappings applied locally (within one segment)}} \\
    \hline
  \end{tabularx} 
  \caption{Overview of the different applications considered in this work and their relation to local and global mappings. A discussion on the difference between these two classes of mappings is given in Section \ref{sec:discussion}. Figure \ref{fig:example_synthetic_butterfly} shows examples of the different applications. For additional details, we refer to the respective section in the experimental evaluation.}
  \label{table:overview_app}
\end{table}

\section{Related works}
\label{sec:related_works}

\subsection{Variational methods}
The connection between energy functionals and the corresponding E-L equation for image processing formulations have been investigated before, we refer to \eg, \cite{aubert-kornprobst:06} and references therein. The total variation (TV) formulation was introduced by Rudin et al.~\cite{Rudin1992}.
The TV-model is well studied, both theoretically and practically, see~\cite{Chan2006} for an overview of the method and its application areas. One major draw-back with the TV-approach for image denoising applications is its tendency to show staircasing effects.  The total generalized variation (TGV) introduced by Bredies et al.~\cite{Bredies:2010:TGV:1958729.1958739}, which includes higher order differentials, does not suffer from these artifacts. The drawback of the approach is that the introduced smoothness of the higher-order derivatives may cause image features to be oversmoothed. Other relevant extensions of the TV model to include additional smoothness constraints are, \eg, \cite{Blomgren97totalvariation,Chen04variableexponent,Bollt2008}.

The extension of TV to variational tensor-based formulations was investigated by Roussos and Maragos \cite{Roussos2010}, Lefkimmiatis et al.~\cite{Lefkimmiatis2013} and Grasmair and Lenzen \cite{Grasmair210}. These approaches consider the structure tensor \cite{diva2:274026,Forstner87} and model the objective functions in terms of the tensor eigenvalues. Roussos and Maragos define their regularization term as $R(u) = \int_\Omega \psi(\mu_1,\mu_2) \; {d\bm{x}}$, where $\psi$ is a strict convex function and $\mu_1$, $\mu_2$ are the eigenvalues of the structure tensor, thus the regularization is only indirectly taking into account the image structure as it ignores the eigenvectors. Similar to Roussos and Maragos, Lefkimmiatis et al. consider the Schatten-norm of the structure tensor eigenvalues. In contrast, Grasmair and Lenzen define the regularization $R(u) = \int_\Omega \sqrt{\nabla u^{\top} A(u) \nabla u} \; {d\bm{x}}$, where $A$ is the positive semidefinite structure tensor with remapped eigenvalues. The objective function is then solved using a finite element method instead of deriving a variational solution. An alternative approach was proposed by Krajsek and Scharr \cite{Krajsek2010} who formulated a linear anisotropic regularization term for tensor-valued image diffusion. The common factor by the aforementioned methods is that they use the structure tensor. Due to the fact that the structure tensor is defined as the integration of the outer product of the image gradient, the divergence theorem is not applicable when deriving the minimizer of the corresponding functional. We show this fact in Section \ref{sec:tensor_diffusion}, Proposition \ref{prop:GF_conv}. To our knowledge, the only tensor-based total variation approach that also considered the eigenvectors, has been studied in~\cite{Astrom2014}, where the lesser known gradient energy tensor (GET) is used instead of the structure tensor. 

In our previous work \cite{diva2:608779} we proposed to filter certain value ranges of the image data, \ie, an image value range adaptive filter. In a related work, \cite{Astrom2013a}, we interpret the image data as observations from a stochastic process and use an oversegmentation of the input image to define the mapping function. The basic approach is to describe each locally homogeneous region using first and second statistical moments to drive a nonlinear filtering process. In this work we combine the mapping function \cite{diva2:608779,Astrom2013a} and the tensor-based method \cite{Astrom2014}, leading to novel results. 

\subsection{Block-matching methods}
In addition to local, \eg, tensor-based, methods, a different paradigm in the denoising literature has emerged called non-local methods. These methods rely on block-matching of neighbourhoods and have been shown to efficiently suppress image noise. The basic idea is to average similar patches in a local neighbourhood, where similarity is measured with respect to some criteria. The denoising methods block-matching by sparse 3-D transform (BM3D)~\cite{DBLP:journals/tip/DabovFKE07} and non-local means (NLM) \cite{1467423} are two such methods. NLM and BM3D are fundamentally similar since both use collaborative filtering. Rather than computing new pixel-centres from matched neighbourhoods as in NLM, BM3D considers patches and adaptively denoises a group of local neighbourhoods. In this context we mention Levin et al. \cite{Levin2012} who studied image denoising  quality w.r.t. patch-size and patch-complexity in view of natural image statistics, they conclude: while the restoration of homogeneous regions is straightforward by increasing the sample size, an increase of patch-size for textured regions has limited impact on the restoration quality as pixels in these regions are weakly correlated. Since, BM3D combines block-matching, linear transform thresholding and Wiener filtering, the noise can only be suppressed efficiently given that sufficiently many similar patches are found. Therefore, as a consequence of inaccurate template matching criteria, oversmoothing may occur in highly textured regions. We compare our framework with BM3D in the experimental evaluation. 

In contrast to patch-matching methods, we present an altogether different view-point to image denoising by focusing on adaptive value range filtering and region specific noise removal. Application areas where a mapping function has shown beneficial are visualization of medical images \cite{diva2:608779} and colour-image denoising \cite{Astrom2013a}. In the present work, we further study this approach by considering a wide range of applications: denoising overexposed images, targeted filtering and denoising of natural images. The potential improvement of image quality and interesting mathematical exposition motivates the subsequent presentation of a theoretical framework. However, before proceeding to the main results, we summarize the background on standard scalar- and tensor valued diffusion methods. 

\section{Review of image diffusion methods}
\label{sec:review_diffusion}
To make this work self-contained we briefly review concepts from calculus of variations and its role in deriving standard diffusion methods. Moreover, we introduce some additional notation and terminology that is used in the remainder of the paper. 

\subsection{Concepts of image diffusion}
The classical approach to image diffusion is to consider functionals, or energy models, consisting of a fidelity term, $F$, and a regularization term, $R$, see \eqref{eq:mainEFR}. Moreover, the energy often takes the form
\begin{equation}  \label{eq:E_diffusion}
  E(u) =  \frac{1}{2}\int_\Omega (u - u^0)^2 \; d {\bm{x}} +  \lambda R(u),
\end{equation}
where ${\bm{x}} \in \Omega$ is the image domain, $u^{0}$ is the observed (noisy) image. The constant $\lambda$ is a positive scalar and determines the effect of the regularization. One well-studied choice of prior is the $p$-Dirichlet energy 
\begin{equation}
  \label{eq:RinEnergy}
  R(u) = \int_\Omega \frac{1}{p}|\nabla u|^p \; d {\bm{x}},
\end{equation}
where $p=1$ corresponds to Total Variation \cite{Rudin1992} and $p>1$ has been studied in. \eg, \cite{Kuijper20091023,diva2:758277}. To minimize $E$, assuming it is convex, one finds the stationary point by solving the corresponding Euler-Lagrange (E-L) equation
\begin{equation}\label{eq:dEL}
  \delta E(u) = 0,
\end{equation} 
and
\begin{equation} \label{eq:EL_min}
  \delta E(u)v =	\lim_{\varepsilon \rightarrow 0} \frac{\partial}{\partial \varepsilon}E (u + \varepsilon v) \quad \mbox{in} \quad  \Omega.
\end{equation}
With homogeneous Neumann boundary condition, $\nabla u \cdot \bm{n} = 0$, the E-L equation of \eqref{eq:E_diffusion}, \eqref{eq:RinEnergy} is the following boundary value problem
\begin{equation}
  \left \{
    \begin{array}{rlll}
      u-u^{0} - \lambda \div{|\nabla u|^{p-2}\nabla u} & = 0 & \mbox{in} & \Omega \\
      \nabla u \cdot \bm{n} & = 0 & \mbox{on} &  \partial \Omega, \\
    \end{array} \right .
  \label{eq:EL_linear}
\end{equation}
where $\bm{n}$ is the normal vector on the boundary $\partial \Omega$. If $p=2$, the E-L equation describes linear (isotropic) diffusion \cite{Koenderink1984} and if $p=1$ it describes the formal TV-formulation of Rudin et al. \cite{Rudin1992}. The $u$ that minimizes \eqref{eq:E_diffusion}, \eqref{eq:RinEnergy} is obtained by solving \eqref{eq:EL_linear}. The corresponding initial value problem is then
\begin{equation}
  \left \{
    \begin{array}{rlll}
      \partial_t u - \div{{|\nabla u}|^{p-2}\nabla u} &= 0  & \mbox{in} & \Omega \\
      \nabla u \cdot \bm{n} &= 0 & \mbox{on} &  \partial \Omega \\
      u(0) &= u^0 & \mbox{on} &   \Omega, \\
    \end{array} \right .
  \label{eq:EL_linear_IVP}
\end{equation}
and results in the diffusion equation \cite{Scherzer2000,DiBenedettoBook,LadyzenskajaSolonnikovUraltceva}. In the case of isotropic filtering, the diffusion equation has a closed form solution. 

\subsection{Nonlinear diffusion}
\label{sec:non-linearDiffusion}
Isotropic filtering does not preserve characteristic image features such as edges and corners due to the rotational symmetry of the Laplacian operator. Therefore, when solving the E-L equation in~\eqref{eq:EL_linear}, the image is blurred. Introducing a diffusivity function, also known as an ``edge-stopping'' function $g : \mathbb{R} \rightarrow \mathbb{R}^+ $ in \eqref{eq:EL_linear}, Perona and Malik \cite{Perona90scale-spaceand} formulated the PDE
\begin{equation}
  \left \{
    \begin{array}{rlll}
      u-u^{0} - \lambda \div{g(|\nabla u|) \nabla u} & = 0 & \mbox{in} & \Omega \\
      \nabla u \cdot \bm{n} & = 0 & \mbox{on} &  \partial \Omega. \\
    \end{array} \right .
  \label{eq:EL_PM}
\end{equation}
The function $g(s)$, with $s=|\nabla u|$, acts to reduce the filtering if the image gradient takes large values. The function's principal behaviour is $g(s)\rightarrow 1$ as $s\rightarrow 0$ and $g(s) \rightarrow 0$ as $s \rightarrow \infty$ and a common choice is the negative exponential function 
\begin{equation}\label{eq:PMg}
  g(s) = \exp(-s^2/k^2),
\end{equation}
where $k>0$ is an edge-stopping parameter. Although this formulation resolves many drawbacks of linear image diffusion, noise is preserved along edges since the filtering is simply ``stopped" for large image gradients. An alternative to using the negative exponential function as the diffusivity function is to let the problem define which edge-stopping function is suitable. In the work on targeted iterative filtering~\cite{diva2:608779} we let the mapping be defined by a user selected sigmoidal function, that compresses the value range of 12-bit image data for the purpose of medical visualization. The regularization term introduced in \cite{diva2:608779} is defined as
\begin{equation}
  R(u) = \int_\Omega |\nabla m(u)|^2 \; d{\bm x},
  \label{eq:EL_TIF}
\end{equation}
where the mapping $m:\mathbb{R}\to\mathbb{R}$ and $m\in\mathcal{C}^2(\Omega)$ is a sigmoidal function. Furthermore, note that $\nabla m(u) = m'(u)\nabla u$. In this case the mapping defines the visualization windows, thus replaces the ad-hoc selection of a potential function, \eg, \eqref{eq:PMg}. The corresponding PDE equation is given by
\begin{equation}\label{eq:L2mapping}
  u - u^0 - \lambda m'(u)\div{m'(u) \nabla u} = 0 \quad \mbox{in} \quad \Omega.
\end{equation}
Since $u^{0}\in L^{2}(\Omega)$ and $m\in {\cal C}^{2}(\Omega)$, it follows that $u\in W^{1,2}(\Omega)$, where $W^{1,2}(\Omega)$ is the Sobolev space equipped with the norm $||u||^{2}_{W^{1,2}(\Omega)}=||u||^{2}_{L^{2}(\Omega)} + ||\nabla u||^{2}_{L^{2}(\Omega)}$ (see, \eg, \cite{gilbarg2001elliptic}). Depending on the definition of $m'$ the function can be understood as the diffusivity function $g$ in~\eqref{eq:PMg}. Note that the PDE \eqref{eq:L2mapping} is equivalent to the formulation in~\cite{diva2:608779} due to the novel Lemma \ref{corollary:div_contraction} in Section \ref{subsec:properties}.

\subsection{Tensor-based diffusion}
\label{sec:tensor_diffusion}
The main deficiency of the methods presented in Section \ref{sec:non-linearDiffusion} is that the filtering is simply stopped close to image features such as edges and corners. To filter parallel to line-structures Weickert~\cite{Weickert96anisotropicdiffusion} introduced a tensor-based anisotropic diffusion scheme. The filtering scheme is defined as the PDE 
\begin{equation} \label{eq:EL_AD}
  u-u^{0} - \lambda \div{D(T) \nabla u} = 0 \quad \mbox{in} \quad  \Omega.
\end{equation}
The adaptivity of the filter is determined by the diffusion tensor $D : \mathbb{R}^{2\times 2} \to \mathbb{R}^{2\times 2}$, constructed from a nonlinear mapping of the structure tensor $T : \mathbb{R}^{2} \to \mathbb{R}^{2\times 2}$ defined as
\begin{equation} \label{eq:structure_tensor}
	T(\nabla u) = w * (\nabla u \nabla u^{\top} ),
\end{equation}
where $*$ is a convolution operator and $w$ is a smooth kernel, \eg, a Gaussian function~\cite{Forstner87,diva2:274026}. The tensor is a windowed second moment matrix, it is positive semidefinite, thus it estimates the local energy distribution and can be thought of being a covariance matrix~\cite{Lindeberg96_direct_computation}. The eigenvector of $T$ corresponding to the largest eigenvalue is aligned orthogonal to the image structure. Therefore, to avoid blurring image structures, the diffusion tensor $D$ is computed as $D(T) = O^{\top} g(\Lambda) O \enspace$ where $g$ is the diffusivity function (\eg, \eqref{eq:PMg}),  $O$ are the eigenvectors and $\Lambda$ has the eigenvalues of $T$ on its diagonal~\cite{Fberg2011}. 

Before continuing we show why a straightforward variational formulation, which gives the PDE \eqref{eq:EL_AD} without the convolution \eqref{eq:structure_tensor}, cannot be extended into a variational formulation for \eqref{eq:EL_AD} with the convolution \eqref{eq:structure_tensor} by inserting the convolution into the regularizer.

This problem has previously been studied by, \eg, Scherzer and Weickert \cite{Scherzer2000}. Here we tackle the energy functional directly and show that the failing component is the application of Green's identity. Now, let $F : \mathbb{R}^{2\times 2} \rightarrow \mathbb{R}^{2 \times 2}$ be a tensor-valued function and define the regularization term
\begin{equation}\label{eq:Ew_func}
  R(u) = \int_\Omega \nabla u^\top F(\nabla u) \nabla u \ d\bm{x}.
\end{equation}
To compute the variational derivative of $R$ we need partial integration in higher dimensions defined by Green's identity \cite{gilbarg2001elliptic}: Let $A$ be a vector field and $v$ once continuously differentiable, then 
\begin{align}\label{eq:EL-div}
  \notag  \int_\Omega A(\nabla u)^\top \nabla v \ d\bm{x} &= \int_{\partial \Omega} v A(\nabla u) \cdot \bm{n} \ dS \\ 
  & \quad - \int_\Omega v \mathrm{div} (A(\nabla u)) \ d\bm{x}.
\end{align}
When $F$ is defined by a nonlinear function and a convolution operator there is no operator $A$ such that Green's formula applies. This is shown in the following proposition. 
\begin{proposition}\label{prop:GF_conv}
  Let $f : \mathbb{R}^{2\times 2} \rightarrow \mathbb{R}^{2 \times 2}$ be a differentiable tensor valued-function and $w$ be a smooth (Gaussian) function. If
  \begin{equation}
    F(\nabla u) = f(w * (\nabla u \nabla u^\top )),
  \end{equation}
  then the corresponding Euler-Lagrange equation of \eqref{eq:Ew_func} does not exist. 
\end{proposition}
\begin{proof}
  Assume that the variational derivative of \eqref{eq:Ew_func} exists. To show this assertion we set
  \begin{equation}\label{eq:s}
    s(\varepsilon) = w * ( \nabla (u+\varepsilon v) \nabla (u+\varepsilon v)^\top ),
  \end{equation} 
  and let
  \begin{equation}
    f(s(\varepsilon)) = F(\nabla(u+\varepsilon v)).
  \end{equation}
  Then from \eqref{eq:EL_min} and \eqref{eq:Ew_func} we have
  \begin{align}\label{eq:Ew}
\notag   \frac{\partial R(u+\varepsilon v)}{\partial \varepsilon} &= \\
    &\hspace{-10mm}\frac{\partial}{\partial \varepsilon} \int_\Omega \nabla (u+\varepsilon v)^\top f(s(\varepsilon)) \nabla (u+\varepsilon v) \ d\bm{x}.
  \end{align}
  From the product rule of differentiation we are interested in the following part from \eqref{eq:Ew}: 
  \begin{align}
\notag    \frac{\partial}{\partial \varepsilon} &  \int_\Omega \nabla u^\top f(s(\varepsilon)) \nabla u \ d\bm{x} \\
    \label{eq:firstEproduct}    & \qquad = \int_\Omega \nabla u^\top \left ( \frac{\partial f(s(\varepsilon))}{\partial \varepsilon}  \right ) \nabla u \ d\bm{x},
  \end{align}
  from which we only need to consider the differentiation of $f$ w.r.t. $\varepsilon$, \ie, 
  \begin{align}\label{eq:dfde}
    \frac{\partial f(s(\varepsilon))}{\partial \varepsilon}  &= \frac{\partial f}{\partial s} \frac{\partial s}{\partial \varepsilon }. 
  \end{align}
  With $s$ as in \eqref{eq:s} we have
  \begin{align}
    \notag  \frac{\partial s}{\partial \varepsilon } &= w * \left ( \frac{\partial}{\partial \varepsilon} \nabla (u+\varepsilon v) \nabla (u+\varepsilon v)^\top \right ) \\
\label{eq:dsde}    &= w * \left (2 \nabla u \nabla v^\top + 2\varepsilon \nabla v \nabla v^\top \right ) .
  \end{align}
  Now, we evaluate the limit $\varepsilon \rightarrow 0$ and obtain from \eqref{eq:firstEproduct}, \eqref{eq:dfde}, \eqref{eq:dsde} the following result
  \begin{align}
    \notag \lim_{\varepsilon \rightarrow 0} & \int_\Omega \nabla u^\top \left ( \frac{\partial f(s(\varepsilon))}{\partial \varepsilon}  \right ) \nabla u \ d\bm{x}   \\
    &=
    \notag \int_\Omega \nabla u^\top \left (  \frac{\partial f(s(0))}{\partial s} \left ( w * \left (2 \nabla u \nabla v^\top \right ) \right ) \right )  \nabla u \ d\bm{x} \\
    &\neq \int_\Omega A(\nabla u)^\top \nabla v \ d\bm{x},
  \end{align}
  for any operator $A(\nabla u)$. Since the element $v$ cannot pass the convolution operator,  Green's formula \eqref{eq:EL-div} can not be applied to give an E-L equation of \eqref{eq:Ew_func}. \qed
\end{proof}
In the next section we present one of the main results of this work: the mapping function-based functional and its E-L equation. 

\section{Mapping-based Image Diffusion}
\label{sec:unifying_functional}
In this section we integrate a nonlinear mapping into a framework of tensor-based image diffusion. 

\subsection{Variational formulation}
In the following, let $m \in \mathcal{C}^2(\Omega)$ be the mapping function which describes the nonlinear mapping of the image value range. Also let $W: \mathbb{R}^n \to \mathbb{R}^{n\times n}$ be a positive semidefinite tensor. Set ${\bm s} = \nabla m(u)$. Let the components of $W(\bm{s})$ be differentiable, then $W_{s_i}({\bm s})$ is the component-wise differentiation of $W({\bm s})$ w.r.t. ${\bm s}$, such that $W_{\bm{s}} \in \mathbb{R}^{n^2\times n}$ is defined as
\begin{equation} \label{eq:Ws}
  W_{\bm s}({\bm s})
  =
  \begin{pmatrix} 
    W_{s_1}({\bm s})  \\
    \vdots \\
    W_{s_n}({\bm s}) 
  \end{pmatrix} . 
\end{equation}
In order to simplify notation in the following calculations we introduce a prescript-notation for matrix contraction, we define
\begin{equation} \label{eq:T_s}
  \prescript{}{\nabla u}{W_{\bm s}({\bm s})}
  =
  \begin{pmatrix} 
    \nabla u^{\top} W_{s_1}({\bm s})  \\
    \vdots \\
    \nabla u^{\top} W_{s_n}({\bm s}) 
  \end{pmatrix}, 
\end{equation}
which contracts $n^2\times n$ matrices such as \eqref{eq:Ws} to an $n \times n$ matrix.

The next theorem generalizes our initial formulation~\cite{Astrom2014} (cf. Theorem 1) which states the relation between a tensor-functional and its corresponding E-L equation. In this work we extend the theorem to dimensions $n \geq 2$ and to include the mapping $m$. Compared with \cite{Astrom2014} we state the theorem in a simplified form, exploiting newly found identities which result in a generalized tensor-based application-driven image denoising framework.
\begin{theorem}
  \label{thm:R}
  Let the regularization term $R$, be given by
  \begin{equation}\label{eq:Rfunctional}
    R(u) = \int_{\Omega} \nabla m(u)^{\top} W(\nabla m(u)) \nabla m(u) \; d{\bm x},
  \end{equation}
  where $u, m \in \mathcal{C}^2(\Omega)$ and $W: \mathbb{R}^n \to \mathbb{R}^{n\times n}$, and $\nabla m \mapsto W$ is a tensor-valued function. Set ${\bm{s}} = \nabla m(u)$, then the corresponding E-L equation is given by
  \begin{equation}\label{eq:ELexp_simplified_2}
    \hspace{0mm}\left \{
      \begin{array}{rl}
        u-u^0 - \lambda \Big ( (m')^2\div{ S_1 \nabla u} \hspace*{5mm} & \\ 
        + \div{ (m')^2 S_2 \nabla u} \Big ) & \hspace*{0mm} = 0 \mbox{ in }  \Omega \\
        {\bm n} \cdot (S_1 + S_2)\nabla u & \hspace*{0mm} =  0     \mbox{ on }  \partial \Omega 
      \end{array} \right .
  \end{equation}
  where
  \begin{subequations}
    \begin{align}
      \label{eq:thm_S1}	S_1 & = 2W({\bm s})  \; + m'\prescript{}{\nabla u}{W_{\bm s}({\bm s})},  \\
      \label{eq:thm_S2}	S_2 & = 2W({\bm s})^{\top} + m'\prescript{}{\nabla u}{W_{\bm s}({\bm s})}.
    \end{align}
  \end{subequations}
\end{theorem}
\begin{proof}
  The proof is a generalization of the results presented in~\cite{Astrom2014} (cf. Theorem 1) and is given in Appendix~\ref{appendix:proof_quad_func}.
\end{proof}

\begin{table*}[t]
  \centering
  \small
  \begin{tabular*}{\linewidth}{@{\extracolsep{\fill}}ll ll}
    \toprule\noalign{\smallskip}
    \textbf{Methods}  &  \textbf{Regularizer} & \textbf{$M$ in E-L equation}         & \textbf{Comment}  \\ 
    \emph{\textbf{Without} mapping}        &                 & $u-u_0 - \lambda \dive (M\nabla u) = 0$    &                   \\
    \noalign{\smallskip}\hline\noalign{\smallskip}
    Isotropic           &  $\displaystyle\int_\Omega |\nabla u|^2 \; d{\bm{x}}$	 &  $I$ \\ 
    \noalign{\smallskip}\hline\noalign{\smallskip}
    TV \cite{Rudin1992}     &   $\displaystyle \int_\Omega |\nabla u| \; d{\bm{x}} $ &  $I/|\nabla u|$   & \\
    \noalign{\smallskip}\hline\noalign{\smallskip}
    EAD \cite{diva2:543914} &   $\displaystyle \int_\Omega \nabla u^{\top} W(\nabla u) \nabla u \; d{\bm{x}}$ & $W + W^{\top} + W_{\bm s}$ \\
    \noalign{\smallskip}\hline\noalign{\smallskip}
    GETV \cite{Astrom2014}         &  $\displaystyle \int_\Omega |\nabla u|(\lambda_1( v\cdot \xi)^2 + \lambda_2(w\cdot \xi)^2)  \; d{\bm{x}}$ & $S$  &  \parbox{3.5cm}{see \eqref{eq:GETV_S} with  $m(u) = u$}\\

    \noalign{\smallskip}\toprule\noalign{\smallskip}
    \emph{\textbf{With} mapping}             &     & $u-u_0 - \lambda m'(u)\dive (M \nabla u) = 0$ \\
    \noalign{\smallskip}\hline\noalign{\smallskip}
    \parbox{1cm}{TIF \cite{diva2:608779} \\ D3 \cite{Astrom2013a}}    &  $\displaystyle \int_\Omega |\nabla m(u)|^2 \; d{\bm{x}}$  & $m'(u)I$ \\ 
    \noalign{\smallskip}\hline\noalign{\smallskip} 
    \noalign{\smallskip}\noalign{\smallskip}
    TVm (\textbf{\emph{novel}})   & $\displaystyle \int_{\Omega} |\nabla m(u)| \; d{\bm{x}} $ &  $I/|\nabla u|$ & See \eqref{eq:TVmEnergy}-\eqref{eq:TVm} \\
    \noalign{\smallskip}\hline\noalign{\smallskip}
    GETVm  (\textbf{\emph{novel}})  &  $\displaystyle \int_\Omega |\nabla m(u)|(\lambda_1( v\cdot \xi)^2 + \lambda_2(w\cdot \xi)^2)  \; d{\bm{x}}$ & $m'(u)S$ &  see \eqref{eq:GETV_S} \\
    \noalign{\smallskip}\toprule\noalign{\smallskip}
  \end{tabular*}
  \caption{Examples of E-L equations for different choices of mapping functions and tensors. We have listed linear diffusion, total variation (TV), extended anisotropic diffusion (EAD), gradient energy total variation (GETV), targeted iterative filtering (TIF) and density driven diffusion (D3). The corresponding boundary conditions can be computed using~\eqref{eq:thm_S1} and~\eqref{eq:thm_S2}. The methods TVm and GETVm are extensions, introduced in this work, of TV and GETV with a mapping function.}
  \label{tab:overview}      
\end{table*}

Depending on choice of mapping $m$ and tensor $W$ in \eqref{eq:Rfunctional}, we show in Table \ref{tab:overview} that the proposed scheme generalizes several diffusion methods from literature. The same table also shows the methods that are subject to further study in this paper. In the subsequent sections we investigate some additional selections of $W$ and $m$ that extend the gradient energy total variation scheme (GETV) presented in~\cite{Astrom2014} to include the mapping. Also we extend the scalar TV formulation to include the mapping function. Before proceeding to specify $W$ and $m$, consider the following properties of the proposed functional and corresponding E-L equation in Theorem~\ref{thm:R}.  

\subsection{Properties}
\label{subsec:properties}
In the general case $\prescript{}{\nabla u}{W_{\bm s}}$ is a non-symmetric tensor, however one may be interested in positive definiteness or other symmetry constraints. For example it is often desirable in orientation estimation that the magnitude response of the estimated orientation is equivalent in vertical and horizontal directions in the image plane. If $S_1$ and $S_2$ in \eqref{eq:ELexp_simplified_2} are symmetric, then it aids the subsequent analysis of otherwise complex expressions to determine positive-definiteness. Therefore, we derive necessary conditions for $W$ to yield a symmetric $\prescript{}{\nabla u}{W_{\bm s}}$ in this section. In the following, we show that a tensor $S$ exists such that the tensors defined by $S_1$ and $S_2$ in the E-L equation are symmetric, \ie, 
\begin{equation}
	\label{eq:EL-equal-brackets}
	S = S_1 = S_2.
\end{equation}
It is easy to identify two additional constraints required for $S$ to be of the form \eqref{eq:EL-equal-brackets} from \eqref{eq:thm_S1},\eqref{eq:thm_S2}, \ie, it is necessary that $W = W^{\top}$ and $\prescript{}{\nabla u}{W_{\bm s}} = (\prescript{}{\nabla u}{W_{\bm s}})^{\top}$. 
Let 
\begin{equation}\label{eq:Wabbc}
  W = \begin{pmatrix} a & b \\ b & c \end{pmatrix},
\end{equation}
then the following corollary state the set of symmetric tensors $W$ that yield a symmetric tensor $S$ in \eqref{eq:EL-equal-brackets}.

\begin{corollary}\label{corr:sym_Ws}
  Let $W \in \mathbb{R}^{2\times 2}$ be the symmetric tensor \eqref{eq:Wabbc} with differential components, then if 
  \begin{itemize}
  \item[i)] $b_{s_1} = a_{s_2}$
  \item[ii)] $c_{s_1} = b_{s_2}$
  \end{itemize}
  are both satisfied, $\prescript{}{\nabla u}{W_{\bm s}}$ is also a symmetric tensor.
\end{corollary}
\begin{proof}
  Since $W$ is symmetric we can write $\prescript{}{\nabla u}{W_{\bm s}}$ as 
  \begin{equation}
    \prescript{}{\nabla u}{W_{\bm s}({\bm s})} =
    u_x
    \begin{pmatrix} 
      a_{s_1} & b_{s_1} \\
      a_{s_2} & b_{s_2}
    \end{pmatrix} 
    + u_y
    \begin{pmatrix} 
      b_{s_1} & c_{s_1} \\
      b_{s_2} & c_{s_2}
    \end{pmatrix}.
  \end{equation}
  Conditions $i)$ and $ii)$ follow directly from the constraint $\prescript{}{\nabla u}{W_{\bm s}} = (\prescript{}{\nabla u}{W_{\bm s}})^{\top}$. \qed
\end{proof}

The solution of $i)$ and $ii)$ is clearly not unique. For example, one solution is given by 
\begin{equation}
  b = \int a_{s_2} \; ds_1 = \int c_{s_1} \; ds_2. 
\end{equation}
If instead $b$ is given, then $a$ and $c$ must satisfy
\begin{align}
  a &= \int b_{s_1} \; ds_2, \\
  c &= \int b_{s_2} \; ds_1,	
\end{align}
which results in
\begin{equation}
  W = \begin{pmatrix} \displaystyle \int b_{s_1} \; ds_2  & b \\ b & \displaystyle \int b_{s_2} \; ds_1 \end{pmatrix},
\end{equation}
to which we can also add any constant symmetric matrix. We note that in the general case it may be difficult, or not interesting, to evaluate Corollary~\ref{corr:sym_Ws}. However, as a theoretical tool it may serve the purpose to derive classes of tensors $S$ that are symmetric.

The following lemmas are particularly useful and enable the reformulation of the E-L equation~\eqref{eq:ELexp_simplified_2} in a single divergence form. In this case we set $M = S_1 = S_2$ where $M$ in contrast to $S$ in \eqref{eq:EL-equal-brackets} is not necessarily a symmetric tensor.

\begin{lemma}\label{lemma:divmu}
  Let $M\in \mathbb{R}^{n\times n}$, then
  \begin{equation}\label{eq:divmu}
    \div{M\nabla u} = \div{M}\nabla u + \tr{M Hu},
  \end{equation}
  where $H$ is the Hessian matrix.
\end{lemma}
\begin{proof}
  We first write the left hand-side of \eqref{eq:divmu} as a sum, then by expanding its partial derivatives we get
  \begin{align}
    \notag    \div{M\nabla u} &= \sum_{ij} \partial_{x_i} (M_{i,j} u_{x_j})  \\    
    \notag    &= \sum_{ij} (\partial_{x_i} M_{i,j}) u_{x_j} + M_{i,j}u_{x_ix_j} \\    
    &= \div{M} \nabla u + \tr{M H u},
  \end{align}
  which shows the result. \qed
\end{proof}

\begin{lemma}\label{corollary:div_contraction}
Given a tensor $M \in \mathbb{R}^{n\times n}$ with differentiable components, the following relation holds
\begin{align}
\notag (m')^2 \div{M \nabla u} + & \div{(m')^2 M\nabla u} \\
\label{lemma:contraction} & \hspace*{5mm} =  2 m' \div{m' M \nabla u} .
\end{align}
\end{lemma}
\begin{proof}
From Lemma \ref{lemma:divmu}, the left hand-side of~\eqref{lemma:contraction} can be written as
\begin{align*}
	& (m')^2 \Big (\div{M} \nabla u + \tr{MHu} \Big )  \\
	& \hspace{1cm}  + \div{(m')^2M} \nabla u + \tr{(m')^2 MHu} \\ 
	& = \Big ( (m')^2 \div{M} + 2m'm''\nabla u^{\top} M \\
        & \hspace{1cm} + (m')^2\div{M} \Big )  \nabla u + 2(m')^2 \tr{MHu} \\ 
	& = 2 \Big ( (m')^2 \div{M} + m'm''\nabla u^{\top} M \Big )\nabla u \\
        & \hspace{1cm} + 2(m')^2 \tr{MHu} \\ 
	& =  2 m' \Big [ \div{m' M}  \nabla u + \tr{m'MHu} \Big ],
\end{align*}
which concludes the proof. \qed
\end{proof}

Observe that depending on $m$, its derivative $m'$ can in the general case obtain negative values. However, due to linearity of the divergence operator, any non-negative factors originating from $m'$ cancel as indicated in the right-hand side of \eqref{lemma:contraction}, \ie,
\begin{equation*}
  (\pm m')\div{(\pm m') M \nabla u}  \iff +  m'\div{ m' M \nabla u}.
\end{equation*}
Assuming that $M$ is at least positive (semi)-definite, then the final discretized E-L equation will have a convergent behaviour.

In particular, with the sufficient condition for tensor symmetry of $W$ and $\prescript{}{\nabla u}{W_{\bm s}}$, and the above simplification of the divergence terms in Lemma~\ref{corollary:div_contraction}, we obtain the E-L equation of the regularization term in Theorem~\ref{thm:R} on the following (considerably) simplified form.
\begin{corollary}\label{corr:EL_contracted}
  Let $W \in \mathbb{R}^{n\times n}$ be a symmetric tensor, then the E-L equation~\eqref{eq:ELexp_simplified_2} is given by
  \begin{equation}\label{eq:ELexp_singleDIV}
    \left \{
      \begin{array}{rlll}
        u-u^0 - \lambda m'\div{m' M \nabla u} & = 0 & \mbox{in} & \Omega \\
        \bm{n} \cdot M\nabla u & = 0    & \mbox{on} & \partial \Omega
      \end{array} \right .
  \end{equation}
  where
  \begin{equation}
    M = 2W({\bm s})   + m'\prescript{}{\nabla u}{W_{\bm s}({\bm s})}.
  \end{equation}
\end{corollary}
\begin{proof}
  Since $W=W^{\top}$ it follows that $S_1 = S_2$, whereby Lemma~\ref{corollary:div_contraction} can be applied to \eqref{eq:ELexp_simplified_2}. \qed
\end{proof}

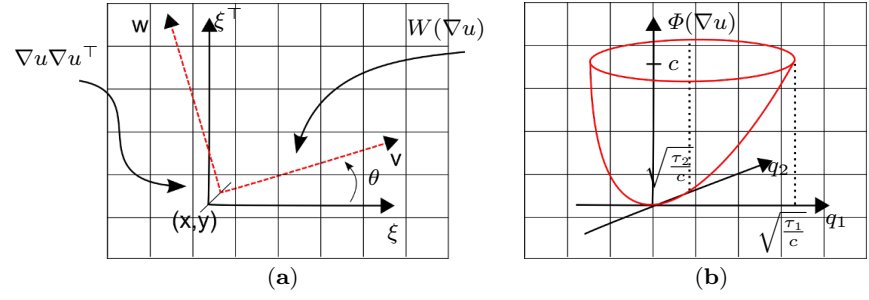
\begin{figure*}[t]
  \centering
  \newcommand{\tabwidth}{1}	
  \begin{minipage}{0.29\linewidth}
    \vspace{-15mm}
    \caption{(a) Illustration of eigenvector basis at coordinate $(x,y)$. Dashed (red) arrows indicate the eigenvectors of $W$ and $S$ and thick (black) arrows $\nabla u\nabla u^{\top}$. (b) Illustration of the paraboloid~\eqref{eq:paraboloid} where we set $c = \tau_1q_1^2 + \tau_2q_2^2$ and $\tau_1, \tau_2 \geq 0$. }
    \label{fig:eigenvectors}
  \end{minipage}
  \hspace{5mm}
  \begin{minipage}{0.65\linewidth}
    \hfill
    \begin{overpic}[width=0.45\textwidth]{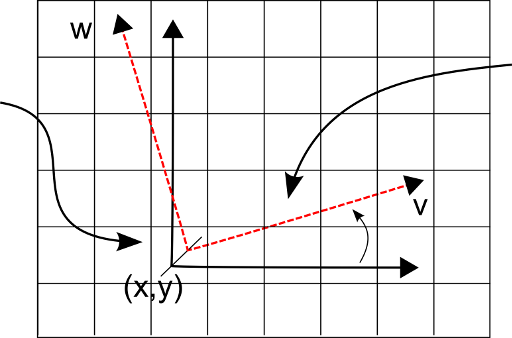}
      \put (35,60) {$\xi^\top$}
      \put (-17,50) {$\nabla u \nabla u^{\top}$}
      \put (75,19) {$\theta$}
      \put (85,58) {$W(\nabla u)$}
      \put (80,5) {$\xi$}
    \end{overpic}
    \hfill \hfill
    \begin{overpic}[width=0.4\textwidth]{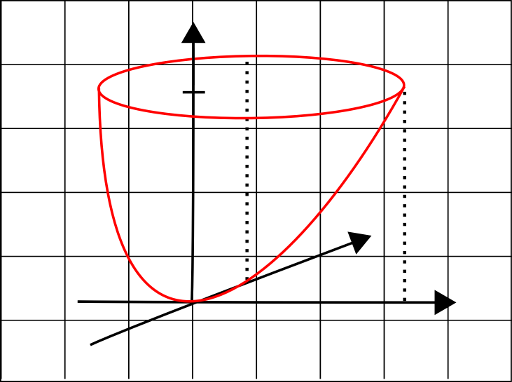}
      \put (42,67) {$\Phi(\nabla u)$}
      \put (42,55) {$c$}
      \put (35,25) {$\sqrt{\frac{\tau_2}{c}}$}
      \put (71,25) {$q_2$}
      \put (68,6) {$\sqrt{\frac{\tau_1}{c}}$}
      \put (88,10) {$q_1$}
    \end{overpic} 
    \hfill \phantom{I} \\
    \begin{tabularx}{\tabwidth\textwidth}{XX}
      \hspace{25mm} (\textbf{a}) & \hspace{25mm} (\textbf{b}) 
    \end{tabularx}
  \end{minipage}
\end{figure*}

\textbf{Convexity} In this section we investigate and discuss convexity of the proposed functional in the case $m(u) = u$ and $W$ is a symmetric and positive semidefinite tensor. If $m$ is a monotonically increasing function, the below results can be generalized using similar arguments, however, for a general function $m$ the description of a convex function remains an open problem. Our approach is based on the eigenvalue decomposition of the tensor $W$. 

Consider the quadratic form $\nabla u^{\top} W(\nabla u) \nabla u$ rewritten in the eigendecomposition of $W$ with eigenvalues $\lambda_1,\lambda_2 \geq 0$ and eigenvectors $v,w$, then 
\begin{align}
\notag	 \nabla u^{\top} W(\nabla u) \nabla u 
	& = \nabla u^{\top} [ \lambda_1 vv^{\top} + \lambda_2 ww^{\top} ] \nabla u  \\
\label{eq:phi_orient}	& \hspace{-10mm} = \lambda_1 v^{\top} [\nabla u \nabla u^{\top}] v + \lambda_2 w^{\top} [\nabla u \nabla u^{\top}] w .
\end{align} 
The product $\nabla u\nabla u^{\top}$ has orthonormal eigenvectors $\xi=(\xi_1,\xi_2)^{\top}$ and $\xi^\perp = (\xi_2, -\xi_1)^{\top}$ such that $P = (\xi, \xi^\perp)$ and $\Lambda$ has the corresponding eigenvalues $\kappa_1$ and $\kappa_2$ on its diagonal. Note that by the spectral theorem \cite{Horn:1985:MA:5509}, the eigendecomposition of $\nabla u\nabla u^{\top}$ is \emph{always} well-defined, \ie, $\xi$ is not singular. This is shown by the generalized definition of $\xi$, \ie, in the case of $|\nabla u| \neq 0$ then $\xi=\nabla u/|\nabla u|$, and in the case of $|\nabla u| = 0$ we let $P = I$ where $I$ is the identity matrix.

In the following, we substitute the eigendecomposition of $\nabla u \nabla u^{\top} = P^{\top}\Lambda P$ into~\eqref{eq:phi_orient}:
\begin{align}
\notag	\nabla u^{\top} W(\nabla u) \nabla u  
\notag	&=   \lambda_1 v^{\top}P \Lambda P^{\top} v + \lambda_2 w^{\top}P \Lambda P^{\top} w  \\
\label{eq:quad_nabla_eig}	&\hspace*{-25mm} =   \lambda_1 v^{\top}P \begin{pmatrix} \kappa_1 & 0 \\ 0 & \kappa_2 \end{pmatrix} P^{\top} v + \lambda_2 w^{\top}P \begin{pmatrix} \kappa_1 & 0 \\ 0 & \kappa_2 \end{pmatrix} P^{\top} w ,
\end{align}
and insert $\kappa_1 = |\nabla u|^2$ and $\kappa_2 = 0$.
Then after rewriting~\eqref{eq:quad_nabla_eig} we obtain 
\begin{equation}\label{eq:phi_mu}
	 \nabla u^{\top} W(\nabla u) \nabla u   
	= \lambda_1 |\nabla u|^2 (v\cdot \xi)^2 + \lambda_2|\nabla u|^2(w\cdot \xi)^2. 
\end{equation}
In \eqref{eq:phi_mu}, $v \cdot \xi = \cos(\theta)$ and $w \cdot \xi = \sin(\theta)$, where $\theta$ is the rotation angle as shown in Figure \ref{fig:eigenvectors} (a). This means that the scalar products define the rotation of $W$ in relation to the image gradient direction. Note that as $W$ describes the local directional information, its eigenvectors will be parallel to the orthonormal eigenvectors of $\nabla u\nabla u^{\top}$, \ie, $v\,||\,\xi$ and $w\,||\,\xi^\perp$ if $\theta = 0$. 
\begin{remark}
	We could have arrived at~\eqref{eq:phi_mu} directly by introducing $|\nabla u|^2$ in the nominator and denominator of \eqref{eq:phi_orient}. However, this approach would not be valid since it introduces a singularity when $|\nabla u|=0$. Instead we exploited properties from the spectral theorem and thus~\eqref{eq:phi_mu} is always well-defined even in the case $|\nabla u| = 0$.  
\end{remark}

The next result shows that the proposed functional is convex when $W$ is symmetric and positive semidefinite with eigenvalues $\lambda_1,\lambda_2 \geq 0$. When the functional is strictly convex it is guaranteed from the theory of convex optimization that its stationary point gives the minimum energy, and thus results in the optimal solution \cite{Boyd:2004:CO:993483}. Note that $m(u) = u$. Now consider the convexity of the proposed functional in the below Corollary. 

\begin{corollary}\label{thm:S}
  Let $m(u) = u$, then the functional, $R$ in \eqref{eq:Rfunctional}, is convex w.r.t. the element $u$. 
\end{corollary}
\begin{proof}
  To prove the convexity of $R$ we write $\Phi(u) = \nabla u^{\top} W(\nabla u) \nabla u$ in terms of the eigenvectors and eigenvalues of $W$. Then, from \eqref{eq:phi_mu} it follows that 
  \begin{align}
    \notag	\Phi(\nabla u) & = |\nabla u|^2(\lambda_1 \xi^{\top} (vv^{\top}) \xi + \lambda_2 \xi^{\top} (ww^{\top}) \xi ) \\
    \notag				   &=  |\nabla u|^2\xi^{\top}(\lambda_1 vv^{\top} + \lambda_2 ww^{\top})\xi \\
    &=   (V\xi)^{\top} \begin{pmatrix} \tau_1 & 0 \\ 0 & \tau_2 \end{pmatrix} V \xi ,
  \end{align}
  where $V = (v, w)$ and $\tau_i(\nabla u) = |\nabla u|^2\lambda_i \geq 0$ for $i = 1,2$. Let $q = V\xi = (q_1, q_2)^{\top}$, then
  \begin{equation}\label{eq:paraboloid}
    \Phi(\nabla u) = \tau_1 q_1^2 + \tau_2 q_2^2, \quad \tau_i  \geq 0, 
  \end{equation}
  is a quadratic form in the basis of orthonormal eigenvectors $V$ and $\tau_i$. The quadratic form is always well-defined due to the spectral theorem. The paraboloid \eqref{eq:paraboloid} has positive curvature everywhere as illustrated in Figure \ref{fig:eigenvectors} (b). Since $u$ is mapped continuously to the paraboloid, $R$ is convex in $u$. \qed
\end{proof}

\vspace{3mm}
\textbf{Extending scalar-valued diffusion} It is straightforward to show that the quadratic form in \eqref{eq:Rfunctional} is related to the standard Euclidean p-norms:  
\begin{proposition}\label{cor:pq-norm}
  Let $W=|\nabla m(u)|^q$ and $q = p - 2$ where $1 \leq p < \infty $, then 
  \begin{equation}\label{eq:pq-norm}
    \nabla m(u)^{\top} |\nabla m(u)|^q \nabla m(u) = | \nabla m(u) |^p .
  \end{equation}
\end{proposition}
\begin{proof}
  Substitute $q=p-2$ in the left-hand side of \eqref{eq:pq-norm}. Since $\nabla m(u)^{\top} \nabla m(u) = |\nabla m(u)|^2$, the right-hand side is obtained after canceling common terms in the quotient.  \qed
\end{proof}

We obtain mapping-based isotropic diffusion by setting $q=0$ in~\eqref{eq:pq-norm}, \ie, it yields the regularization term \eqref{eq:EL_TIF} studied in \cite{diva2:608779}. In this case the E-L equation is given by~\eqref{eq:L2mapping}. Additionally, from \eqref{eq:pq-norm}, if we select $q=-1$ then
\begin{equation}\label{eq:TVmEnergy}
	R(u) = \int_\Omega |\nabla m(u)| \; {d\bm{x}}. 
\end{equation}
The regularization term \eqref{eq:TVmEnergy} can also be interpreted in relation to the standard Rudin-Osher-Fatemi (ROF) TV-model \cite{Rudin1992}, see definition in~\cite{aubert-kornprobst:06}, pp. 39-43, but now with a mapping function. 

The formulation \eqref{eq:TVmEnergy} can also be obtained formally from the initial functional \eqref{eq:Rfunctional} in Theorem \ref{thm:R}. Set 
\begin{equation}\label{eq:W_TVm}
	W(\nabla m(u)) = \frac{1}{|\nabla m(u)|} I,
\end{equation}
then the differentiation with respect to each component of $W$ yields
\begin{equation}
	\prescript{}{\nabla u}{W_{\nabla m(u)}(\nabla m(u))} = - \frac{1}{(m')^2|\nabla u|^3} \nabla u \nabla u^{\top}. 
\end{equation}
Since $W$ in~\eqref{eq:W_TVm} is obviously symmetric, Corollary~\ref{corr:EL_contracted} is applicable and the resulting E-L equation is given by
\begin{equation}  \label{eq:TVm}
  \left \{
    \begin{array}{rlll}
    u-u^0- \lambda m'\div{ \dfrac{\nabla u}{|\nabla u|}}   & = 0 & \mbox{in} & \Omega \\
      {\bm n} \cdot \nabla u & = 0    & \mbox{on} & \partial \Omega \enspace .
    \end{array} \right .
\end{equation}

The standard isotropic diffusion and total variation formulations are obtained by setting $m(u) = u$, which yields $m'(u) = 1$ in \eqref{eq:pq-norm}. Thus the standard methods are special cases of our presented framework. The introduction of the mapping function adds an additional nonlinear component that controls the amount of filtering based on some prior knowledge of the image value range.

Before continuing to the tensor-valued formulations of Theorem~\ref{thm:R}, we recall the gradient energy tensor \cite{diva2:269100} and present an empirical evaluation of the tensor compared to the structure tensor \cite{diva2:274026}. In \cite{diva2:543914} we proved that the connection between the functional and its E-L is not preserved when the structure tensor is subject to a nonlinear mapping of its eigenvalues. This has also been observed by others \cite{Krajsek2010,Black1996}, furthermore, an explicit derivation of this fact is given in Proposition \ref{prop:GF_conv}, Section \ref{sec:tensor_diffusion} of this work. Since the GET does not have a post-convolution of its tensor components we remove some obstacles to preserve the connection between the energy and the corresponding E-L equation.

In the next section we review the gradient energy tensor before introducing the gradient energy total variation scheme with a mapping function. 

\section{Gradient Energy Tensor}
\label{sec:gradient_energy_tensor}
In this section we review the gradient energy tensor \cite{diva2:269100} (GET) and analyse how sensitive it is with regards to orientation estimation compared with the structure tensor. Our analysis show that the GET can approximate the structure tensor for image denoising.  

\subsection{Background and definition}

\begin{figure*}[t]
  \centering
  \begin{minipage}{0.29\linewidth}
    \vspace{-48mm}
    \caption{Visualization of estimated local orientations for (a) the structure tensor and (b) the gradient energy tensor.}
    \label{fig:tensors}
  \end{minipage}
   \begin{minipage}{0.70\linewidth}
     \hfill  \includegraphics[width=0.49\textwidth]{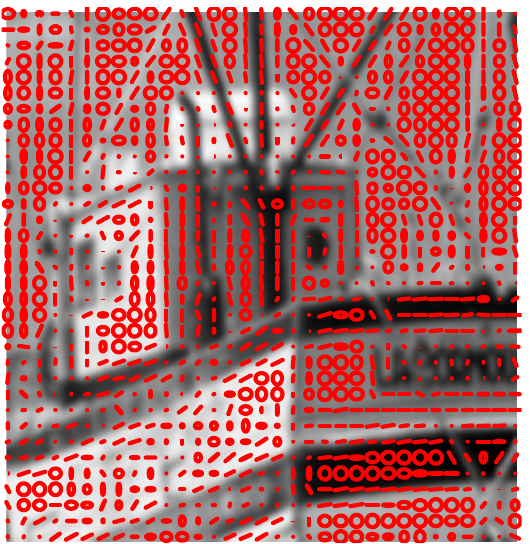} \hfill \hfill
     \includegraphics[width=0.49\textwidth]{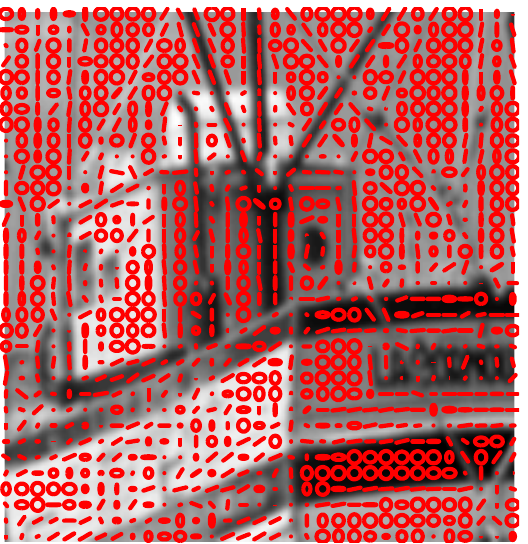} \hfill \phantom{I}  \\
     \hspace*{0mm} \hfill (\textbf{a}) Remapped structure tensor \hfill \hfill (\textbf{b}) Remapped energy tensor \hfill \phantom{I} \\
   \end{minipage}
\end{figure*}

The origin of the gradient energy tensor comes from the 1-D energy operator \cite{275632}, and in particular, from the observation that the energy of a signal can be estimated from its squared magnitude~\cite{115702}. The gradient energy tensor is a symmetric tensor, like the structure tensor, and it can be shown that GET is phase invariant, orientation equivariant second order tensor~\cite{Felsberg04poidetection,nordberg95}. The GET tensor is lesser known in the computer vision community and it was first introduced by Felsberg and K{\"o}the \cite{diva2:269100}. The motivation for considering the GET in this work, as opposed to the structure tensor, is three-fold (I) the GET can potentially enable real-time image diffusion implementations which is not possible with the structure tensor (due to the post-convolution of its tensor-components, for details see \cite{AstroemACCV}), (II) the GET preserves the relation between the energy and the E-L equation when the tensor is positive semi-definite and (III) to promote further studies on the tensor. Similarly to the structure tensor, GET is of rank 2 and determines the directional energy distribution of the signal gradient.

The classical GET is defined in terms of the image data in $u$~\cite{diva2:269100}. Here we use an alternative, but equivalent, formulation of GET expressed in the image gradient. Let $Hm(u) = \nabla (\nabla m(u)^{\top})$ be the Hessian and $\nabla(\Delta m(u)) = \nabla (\nabla^{\top} \nabla m(u))$, where $\Delta$ is the Laplace operator, then we define GET as
\begin{align}\label{eq:GET}
  \notag	& \mathrm{GET}(\nabla m(u)) = Hm(u)Hm(u) \\
  &  \hspace{1mm} - \frac{1}{2}\Big(\nabla m(u) [\nabla \Delta m(u)]^{\top} + [\nabla \Delta m(u)]\nabla m(u)^{\top} \Big).
\end{align} 
Since the GET components are expressed in terms of the image gradient, Theorem~\ref{thm:R} is applicable. In contrast to the structure tensor~\eqref{eq:structure_tensor}, the gradient energy tensor~\eqref{eq:GET}  does \emph{not} (necessarily) require a convolution operator to form a rank 2 tensor. 

Example responses from the two tensors are illustrated in Figure \ref{fig:tensors}. Note that due to the convolution operator, the structure tensor is not sensitive to structures smaller than the width of the averaging filter used to compute it (in this case the standard deviation of the Gaussian filter was set to 1). The presence of second and third-order derivatives in GET makes it slightly more sensitive to noise. However, it allows us to capture orientation of structures not possible to detect using the structure tensor.

\subsection{Spectral properties of GET}
An investigation of the positivity of the 1-D energy operator was presented in~\cite{275632}. In the 2-D case, the positivity of the operator is reflected in the sign of the eigenvalues. Let the components of the GET be $a$, $b$ and $c$, \ie, 
\begin{equation}\label{eq:GET_components}
  \mathrm{GET} = \begin{pmatrix} a & b \\ b & c \end{pmatrix}, 
\end{equation}
then the GET is positive semidefinite if the following lemma is satisfied.
\begin{lemma}\label{lemma:1}
  The GET is positive semidefinite if 
  \begin{equation}\label{eq:cond_get}
    \tr{Hm(u)Hm(u)} - \nabla m(u)^{\top} \nabla \Delta m(u) \geq \sqrt{l},
  \end{equation}
where $l = \tr{\mathrm{GET}}^2 - 4\det(\mathrm{GET}) \geq 0 $.
\end{lemma}
\begin{proof}
  Since GET is symmetric it has real eigenvalues. Thus by its eigenvalue decomposition it is sufficient to show that $\tr{\mathrm{GET}} \geq \sqrt{l}$ (\ie,  \eqref{eq:cond_get}) in order for GET to be positive semidefinite. $l$ is necessarily positive since $l = (a-c)^2 + 4b^2 \geq 0$. \qed
\end{proof}
The GET-tensor can attain zero-values despite the gradient is not zero, and in some cases the eigenvalues may be negative. Here, we set the negative eigenvalues to positive in the $\mathrm{GET}^+$ tensor. We define
\begin{equation}\label{eq:GET+}
  \mathrm{GET}^+(\nabla m(u)) = vv^{\top}|\mu_1| + ww^{\top}|\mu_2|,
\end{equation}
and $v$, $w$ are the eigenvectors and $\mu_1$, $\mu_2$ eigenvalues of the GET respectively. 
In the standard approach of anisotropic diffusion one considers the diffusion tensor, which is the structure tensor with transformed eigenvalues as described in Section \ref{sec:tensor_diffusion}. We do a similar mapping of \eqref{eq:GET+} by using the negative exponential function, \ie, we have
\begin{align}
\notag	D(\nabla m(u)) &= \exp(-\mathrm{GET}^+(\nabla m(u))/k^2) \\
\label{eq:Dget+}   &=  \lambda_1 vv^{\top} +  \lambda_2 ww^{\top},	
\end{align}
and $v,w$ are the same eigenvectors as in~\eqref{eq:GET+} and the eigenvalues $\lambda_{1,2} = \exp(-|\mu_{1,2}|/k^2)$. Now $\mathrm{GET}^+$, is positive semidefinite as $\exp$ is the matrix exponential function obtained by a Taylor series expansion. Figure~\ref{fig:tensors} shows a visual comparison between the remapped energy tensor and the diffusion tensor. Visually, the local orientation estimation appears nearly identical between the two tensors. Next, we quantify this assertion. 

\begin{figure*}[t]
    \centering
        	\includegraphics[width=0.32\textwidth]{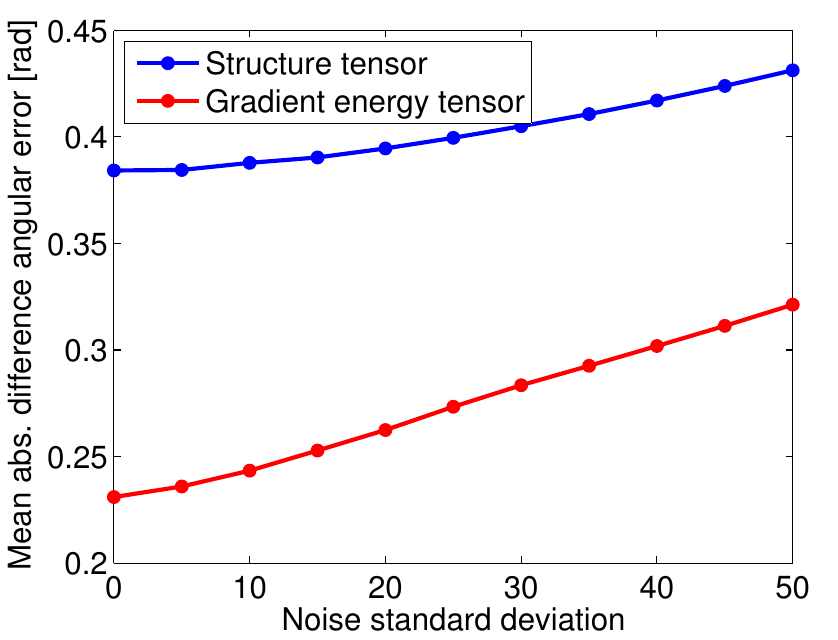} 
        	\includegraphics[width=0.31\textwidth]{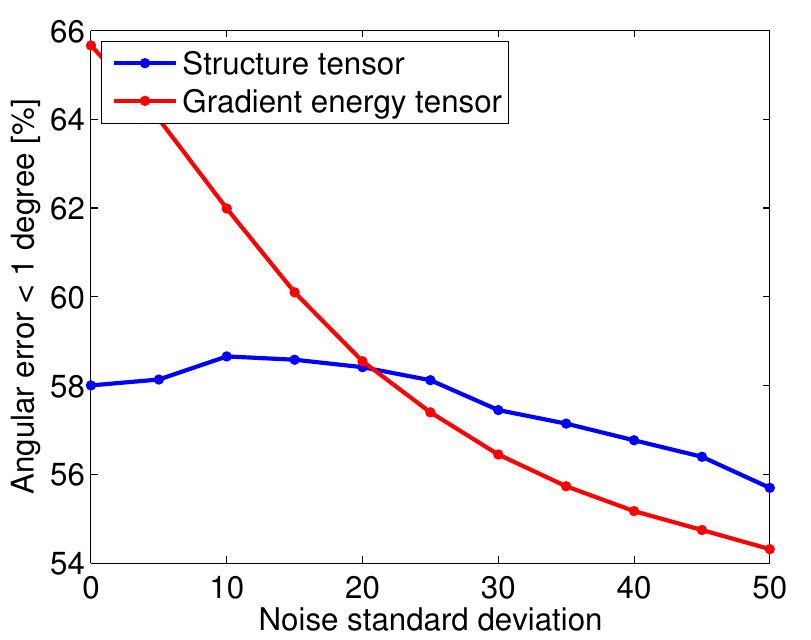}
        	\includegraphics[width=0.32\textwidth]{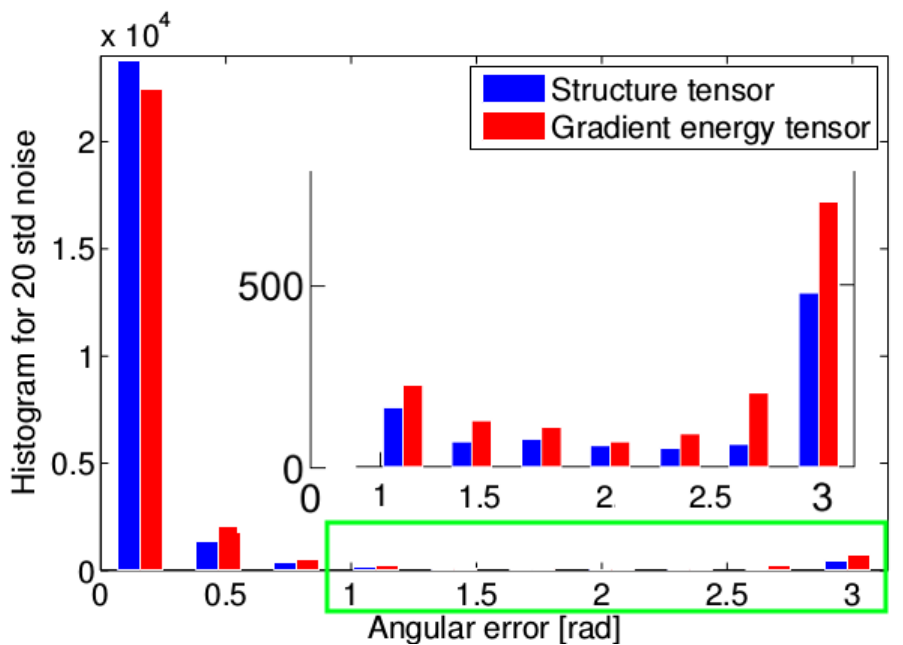} \\
        	\hfill (\textbf{a}) \hfill \hfill (\textbf{b}) \hfill \hfill  (\textbf{c}) \hfill  \phantom{I} \\
    \caption{Results of sensitivity analysis of orientation estimation by using the structure tensor and the gradient energy tensor. We used the radially symmetric pattern (a) for different noise levels in Figure~\ref{fig:pattern_est}. (a) shows the mean absolute angular error. (b) illustrates the percentage of pixels which has an angular error smaller than 1 degree. (c) shows the histogram of angular errors divided into 10 bins from 0 to $\pi$ radians. See text for details. }
    \label{fig:angle}
    \vspace{5mm}
    \hfill
    \includegraphics[width=0.20\textwidth]{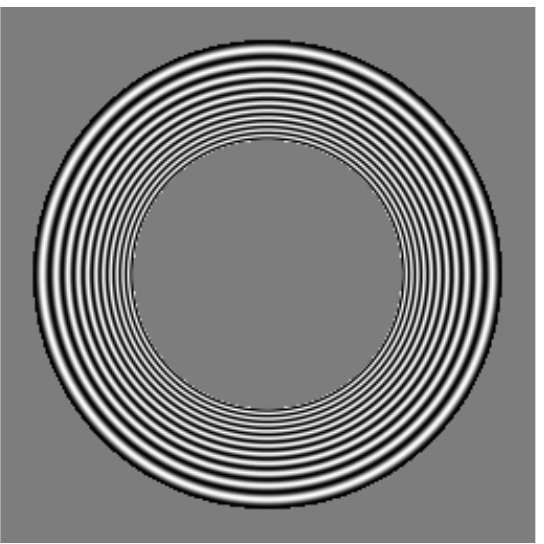} \hfill\hfill
    \hfill\hfill \includegraphics[width=0.23\textwidth]{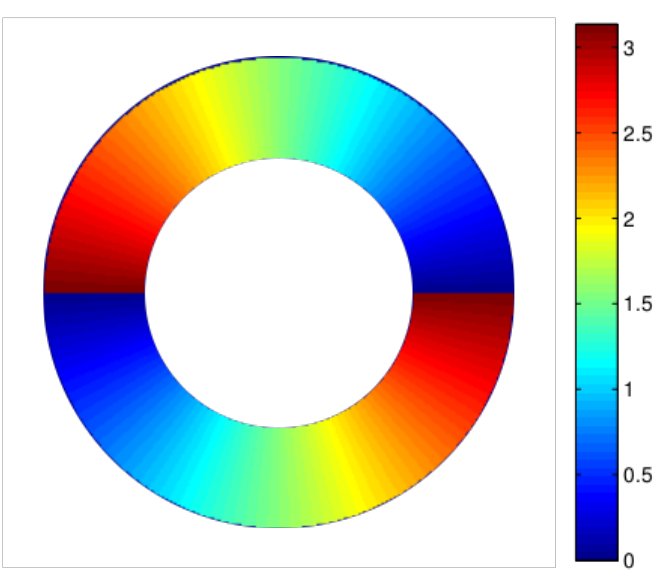} \hfill\hfill
    \hfill\hfill \includegraphics[width=0.23\textwidth]{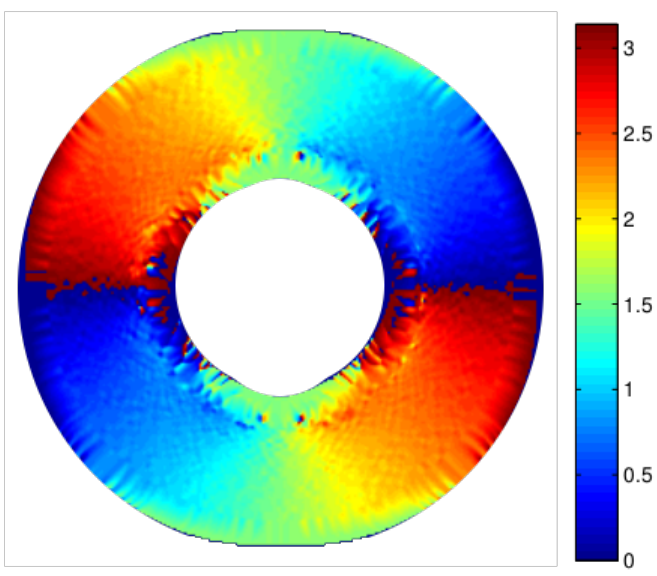}\hfill\hfill
    \hfill\hfill \includegraphics[width=0.23\textwidth]{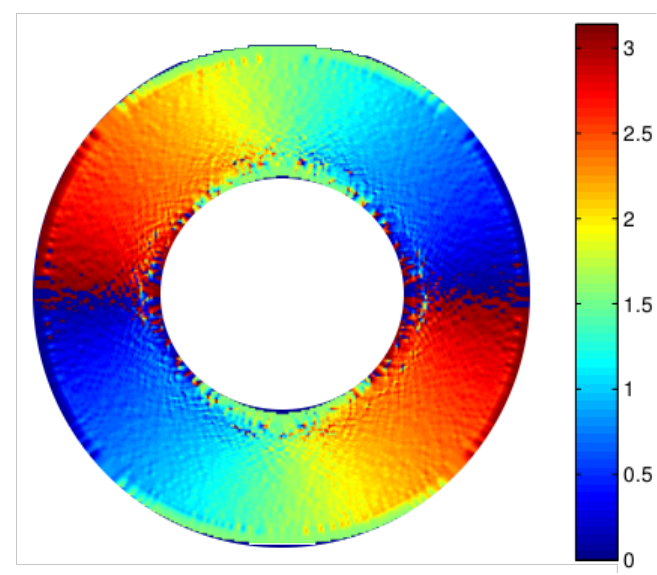}\hfill \\
    \hfill (\textbf{a}) \hfill \hfill (\textbf{b}) \hspace{8mm} \hfill \hfill (\textbf{c}) \hfill \hfill (\textbf{d}) \hfill \phantom{I} \\
    \caption{(a) Original, noise free circular pattern. (b) Ground truth orientation. Orientation in radians at standard deviation 20 of noise, (c) Structure tensor and (d) Gradient energy tensor.}
    \label{fig:pattern_est}
\end{figure*}

\subsection{Analysis}
In this part we discuss the effect of noise perturbation on the tensor's estimated orientations. We consider a radially symmetric test-pattern shown in Figure~\ref{fig:pattern_est} (a). As a baseline for the evaluation we use the noise-free angles that have also been used to compute the test-pattern image. For each tensor and considered noise level, we compute the local neighbourhoods dominant distribution angle by projecting the tensor components on a vector $z$ in the complex plane~\cite{books/daglib/0095066}, 
\begin{equation}
	z = a-c + i2b,
\end{equation}
where $a$, $b$ and $c$ are components of a symmetric $2\times 2$ matrix. We compare the mean absolute difference of the angular error in Figure~\ref{fig:angle} (a) for different noise levels. In (b) we show the percentage of pixels that have an angular error smaller than 1 degree. It is evident that the two tensors are marginally different. As the noise-level increases beyond standard deviation 20, the structure tensor performs better due to the post-convolution of the tensor components. At standard deviations 0 and 10 of noise the difference is approximately 8\% and 3\% respectively. Our third evaluation is a comparison of the histogram error quantized to 10 bins in the range $0-\pi$ at the intersection point seen in (b), \ie, at the noise level of standard deviation 20. The histogram is shown in (c) and we see that the error distribution is indeed similar for the two tensors. We also see that for the angles less than $0.5$ rad the structure tensor shows higher errors than the structure tensor due to the post-convolution of its components. However, considering the right most bins, the structure tensor has fewer pixels with large errors than the gradient energy tensor.

Objectively, the gradient energy tensor has more large magnitude errors than the structure tensor. However, this is expected since the energy tensor lacks a post-convolution of the tensor components to compensate for the noise. Nevertheless, considering that the difference between the two tensors mean angular error is small, as indicated in (a), it still suggests that the gradient energy tensor can be used in a diffusion framework, a claim that we further support in our numerical experiments, and was shown in \cite{AstroemACCV}. In the following sections we use the gradient energy tensor to model a tensor-based total variation framework. 

\section{Mapping-based Gradient Energy Tensor Total Variation}
\label{sec:GETVm}
In this section we extend our previous work~\cite{Astrom2014,diva2:543914} on the gradient energy tensor total variation (GETV). In particular we introduce the mapping to GETV and denote the formulation as GETVm.

\subsection{The objective function}
The below corollary gives the definition of the GETVm scheme. The tensor that steer the filter is the transformed gradient energy tensor, \ie, \eqref{eq:Dget+}. In the following, let $W$ be a tensor such that the following relation holds
\begin{equation}\label{eq:DW}
  D(\nabla m(u)) = W(\nabla m(u)) |\nabla m(u)|^q,
\end{equation}
where $D$ is the tensor in \eqref{eq:Dget+} and has eigenvalues $\lambda_{1}$, $\lambda_2$ and eigenvectors $v,w$. Then~\eqref{eq:DW} can be reformulated into a modified expression of~\eqref{eq:phi_mu}, \ie,
\begin{align}
  & \notag \nabla m(u)^{\top}  W(\nabla m(u)) \nabla m(u) \\ 
  \label{eq:func_q} &= \lambda_1 |\nabla m(u)|^{2-q} (v \cdot \xi)^2 + \lambda_2 |\nabla m(u)|^{2-q} (w \cdot \xi)^2.
\end{align}
where the left hand side is the integrand of \eqref{eq:Rfunctional}. Now we give the following statement

\begin{corollary}
\label{corr:R}
Set ${\bm s} = \nabla m(u)$ and let $q=1$ in \eqref{eq:func_q}. Then the gradient energy total variation regularization term is given by 
\begin{equation}
	R(u) = \int_\Omega |\nabla m(u)|( \lambda_1 (v \cdot \xi)^2 + \lambda_2 (w \cdot \xi)^2 ) \; {d\bm{x}},
\end{equation}
and the tensor $\prescript{}{\nabla u}{W_{\bm s}}$ in~\eqref{eq:T_s} is
\begin{align} \label{eq:corr1_Ts}
\notag	\prescript{}{\nabla u}{W_{\bm s}({\bm s})} & = 
	-\frac{1}{|s|^3}
	 \begin{pmatrix} 
		s_1 \nabla u^{\top}  D({\bm s}) \\ 
		s_2 \nabla u^{\top}  D({\bm s})
	 \end{pmatrix} \\
	 & \hspace{1.5cm} 	 
	+
	\frac{1}{|s|}
	\begin{pmatrix}  \
		\nabla u^{\top}  D_{s_1}({\bm s}) \\ 
		\nabla u^{\top}  D_{s_2}({\bm s})
	 \end{pmatrix},
\end{align}
where $D$ is defined in~\eqref{eq:Dget+}.
\end{corollary}
\begin{proof}
  See Appendix \ref{app:proofGETV}. \qed
\end{proof}

Due to the definition of the tensor $D$ in \eqref{eq:DW} the resulting functional describes an anisotropic total variation filtering method. The tensor acts as a weight limiting the TV filtering perpendicular to image structures. In a homogeneous region $D(\nabla m(u)) \approx I$ the regularization term in the functional reduces to $R(u) = \int_\Omega |\nabla m(u)| \; {d\bm{x}}$ previously described in~\eqref{eq:TVmEnergy}. It is obvious that the GET-tensor will be isotropic in homogeneous regions. 

In the section below we show how to derive the E-L equation corresponding to the selection of $W$ in~\eqref{eq:DW} that yields the diffusion scheme.

\subsection{Diffusion formulation}
\label{sec:diffusion_formulation}
To obtain the diffusion formulation of GETVm, note that $W$ in~\eqref{eq:DW} is symmetric. Then, according to Corollary~\ref{corr:EL_contracted}, it is straightforward to show that the tensors in the E-L equation~\eqref{eq:ELexp_simplified_2} are given by $S = S_1 = S_2 = 2W + \prescript{}{\nabla u}{W_{\bm s}}m'$. With $\prescript{}{\nabla u}{W_{\bm s}}$ given in~\eqref{eq:corr1_Ts} we obtain 
\begin{align}
  \notag	S&= \bigg \{ \frac{2D(s)}{|s|} 
  -\frac{1}{|s|^3}
  \begin{pmatrix} 
    s_1 \nabla u^{\top}  D(s) \\ 
    s_2 \nabla u^{\top}  D(s)
  \end{pmatrix} \\
\notag  & \hspace{30mm} +
  \frac{1}{|s|}
  \begin{pmatrix}  \
    \nabla u^{\top}  D_{s_1}(s) \\ 
    \nabla u^{\top}  D_{s_2}(s)
  \end{pmatrix} \bigg \} \bigg |_{\bm{s}=m(u)}\\
\label{eq:GETV_S}  &= \frac{M}{m'(u)|\nabla u|},
\end{align}
where 
\begin{align}
\notag  M &=  
  2D(s)
  -\dfrac{1}{|\nabla u|^2}
  \begin{pmatrix} 
    u_x \nabla u^{\top}  D(s) \\ 
    u_y \nabla u^{\top}  D(s)
  \end{pmatrix}  \\
\label{eq:GETV_M}  & \hspace{30mm} +
  \begin{pmatrix}  \
    \nabla u^{\top}  D_{s_1}(s) \\ 
    \nabla u^{\top}  D_{s_2}(s)
  \end{pmatrix}. 
\end{align} 
Now, from Corollary \ref{corr:EL_contracted}, \eqref{eq:ELexp_singleDIV} we get the E-L equation
\begin{equation}\label{eq:E-L_3terms}
  \hspace{0mm}\left \{
    \begin{array}{rl}
      u - u^0 -  \lambda m'\mathrm{div} \Bigg (
      M
      \dfrac{\nabla u}{|\nabla u|} \Bigg )  & \hspace*{-0mm} = 0 \mbox{ in }  \Omega \\
      \hspace{0mm} {\bm n} \cdot M\nabla u & \hspace*{-0mm} =  0     \mbox{ on }  \partial \Omega 
    \end{array} \right .
\end{equation}

The difference to TV is the introduction of a mapping function and the tensor in front of $\nabla u / |\nabla u|$ in the divergence. These weights control the anisotropy of the TV scheme: reducing the filtering orthogonal to image structures and accelerating the filtering parallel to image structures. The orientation estimation is achieved using the GET-tensor with positive eigenvalues, defined in~\eqref{eq:GET+}. However, note that using GET is not unique, any tensor, also the diffusion tensor constructed from the structure tensor may be applicable in this framework. 

The derivatives of $D$ with respect to $\partial_x m(u)$ and $\partial_y m(u)$ in \eqref{eq:GETV_M} are obtained by first computing the eigenvectors $v$ and $w$ and eigenvalues $\lambda_1$, $\lambda_2$ of $D$.  The differentiation of $D$ in \eqref{eq:DW} w.r.t. $s_1 = \partial_x m(u)$ reads
\begin{align}\label{eq:D_ux}
\notag	D_{s_1} & = (\partial_{s_1}vv^{\top})\lambda_1 + vv^{\top}(\partial_{s_1}\lambda_1) \\
		 & \hspace{5mm} + (\partial_{s_1}ww^{\top})\lambda_2 + ww^{\top}(\partial_{s_1}\lambda_2),
\end{align}
and $\lambda_i = \exp(-|\mu_i|/k^2)$, $i = 1,2$ with the corresponding orthonormal eigenvectors ($b \neq 0$)
\begin{align}\label{eq:eigenvector_v}
	v
	=
	\begin{pmatrix}
		v_1 \\
		v_2
	\end{pmatrix}
	=
	\frac{1}{|\tilde{v}|}
	\tilde{v},
	\;\; \mbox{where} \;\;
	\tilde{v} = 
	\begin{pmatrix}
		2b \\
		c-a+\sqrt{l}
	\end{pmatrix}, 
\end{align}
and
\begin{align}\label{eq:eigenvector_w}
	w
	=
	\begin{pmatrix}
		w_1 \\
		w_2
	\end{pmatrix}
	=
	\frac{1}{|\tilde{w}|}
	\tilde{w},
	\;\; \mbox{where} \;\;
	\tilde{w} = 
	\begin{pmatrix}
		2b \\
		c-a-\sqrt{l}
	\end{pmatrix},
\end{align}
and $l$ is given in Lemma~\ref{lemma:1}. Appendix \ref{app:eigenD} further details the components of \eqref{eq:D_ux}. In the case $b=0$ the eigenvalues are given on the diagonal of GET and the orthonormal eigenvectors are given by the identity matrix. The implication of $b=0$, is that derivatives with respect to $\partial_{\bm{s}} vv^{\top}$ and $\partial_{\bm{s}} ww^{\top}$ in~\eqref{eq:D_ux} vanish. 

\subsection{Discretization}
\label{sec:numerical_aspects}
The PDEs we have introduced in this work can all be expressed on the form 
\begin{equation}\label{eq:pdeM}
  u-u_0 - \lambda m'(u)\dive \left (M  \frac{\nabla u}{|\nabla u|} \right ) = 0,
\end{equation}
where $\nabla u \mapsto M$ and $M \in \mathbb{R}^{2 \times 2}$ is a tensor different for each application (selection) of $W$ and $m$. Different selections of $W$ are outlined in Table \ref{tab:overview}. 

To approximate the PDEs we use a simple, but generic, PrimalDual approach of Chan, Golub and Mulet \cite{doi:10.1137/S1064827596299767}. The idea is straight-forward and the aim is to avoid explicit computation of the quotient in \eqref{eq:pdeM}, therefore they introduce a dual variable, $d$. In the following, let $\lambda, \beta, \gamma$ and $\tau$ be small positive constants. In our case we define the dual variable as
\begin{equation}
  d = \frac{M\nabla u}{\sqrt{|\nabla u|^2 + \beta}},
\end{equation}
where we introduce $\beta$ in the denominator. This gives the nonlinear equation system
\begin{subequations}
  \begin{align}
    \label{eq:P}   u-u_0 - \lambda m'(u)\dive (d) &= 0 \\
    \label{eq:D}  d\sqrt{|\nabla u|^2 + \beta} - M\nabla u &= 0. 
  \end{align}
\end{subequations}
We update of the primal variable $u$ as 
\begin{equation}\label{eq:update_u}
  u^{k+1} = u^k + \tau \left ( u^k-u_0  - \lambda m'(u^{k}) \mathrm{div} (d^{k+1}) \right )  ,
\end{equation}
where $k$ is an iterator. To solve for the dual variable in \eqref{eq:D} we introduce the iterator $i$, and use the current $u^k$ as initial value, \ie, we have 
\begin{align}
\begin{dcases}
  d^{k}_{i+1} &= d^k_i + \gamma \big ( d^k_i \sqrt{|\nabla u^k|^2 + \beta} - M_k\nabla u^k \big )\\
  d^{k+1} &\leftarrow d^{k}_{i+1}
\end{dcases}
\end{align}
The divergence term in \eqref{eq:P} is discretized using forward finite differences \cite{JWeickert_Disc}, \ie,
\begin{align} \label{eq:div_expanded}
  \dive (d^{k+1}) = \partial^+_x d^{k+1} + \partial^+_y d^{k+1} , 
\end{align}
where 
\begin{subequations}
\begin{align}
  \partial_x^+ u & = u(x+1,y) - u(x,y) \\
  \partial_y^+ u & = u(x,y+1) - u(x,y).
\end{align}
\end{subequations}

\begin{figure}[t]
  \centering
  \includegraphics[width=0.49\textwidth]{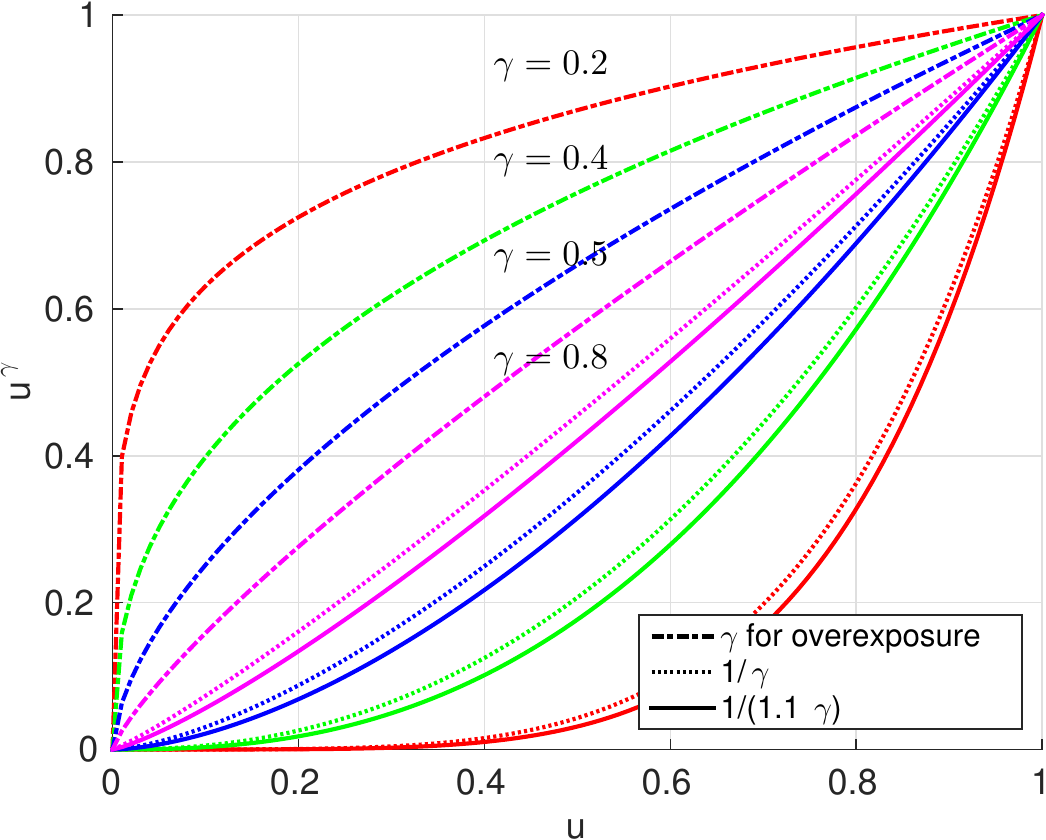}
  \caption{Mapping function for gamma correction. Illustration of different gamma parameters (dash-dotted), corresponding compensation curves (dashed) and the compensation curve with 10\% error (thick). The dash-dotted curve is used to obtain the noisy images and the thick curve is used to obtain the denoised results in Figure \ref{fig:gc}.}
  \label{fig:gamma_mapping}
\end{figure}	

\begin{table}
  \centering
  \renewcommand*{\arraystretch}{1.3}
  \begin{tabularx}{0.65\linewidth}{r|cc}
    & \textbf{ISO/m}, \textbf{TV/m}  & \textbf{GETV/m} \\
    \hline
    $\lambda$ & 0.1          & 0.05   \\ 
    \hline
    $\tau$ & \multicolumn{2}{c}{$\lambda/8$}  \\
    $\beta$ & \multicolumn{2}{c}{$10^{-12}$}  \\
    $\gamma$ & \multicolumn{2}{c}{1}  \\
    $k$ & \multicolumn{2}{c}{100}  \\
  \end{tabularx} 
  \caption{Parameters of the numerical scheme were fixed for all applications. ``\textbf{/m}'' denotes with and without mapping. The abbreviations stand for isotropic diffusion (ISO), isotropic diffusion with a mapping function (ISOm) \cite{Astrom2013a},  total variation (TV) \cite{Rudin1992}, total variation with a mapping function (TVm), see \eqref{eq:TVm}, gradient energy tensor total variation (GETV) \cite{Astrom2014} and GETV with a mapping function (GETVm), see \eqref{eq:E-L_3terms}.}
  \label{table:parameters}
\end{table}

\renewcommand{\figwidth}{0.24}
\begin{figure*}[t]
  \newcommand{\tabwidth}{1}	
  \centering
  \begin{tabularx}{\tabwidth\textwidth}{XX XX}
    \textbf{Noisy} & \textbf{Filter then GC} & \textbf{GC then filter} & \textbf{Mapping-based filter} \\
  \end{tabularx}\\[0mm]
  \begin{tabularx}{\tabwidth\textwidth}{XX XX}
    $\gamma = 0.8$
  \end{tabularx}\\
  \includegraphics[width=\figwidth\textwidth]{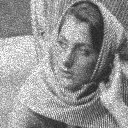}
  \includegraphics[width=\figwidth\textwidth]{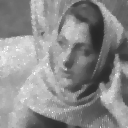}
  \includegraphics[width=\figwidth\textwidth]{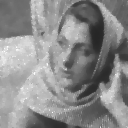}
  \includegraphics[width=\figwidth\textwidth]{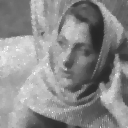} \\
  \begin{tabularx}{\tabwidth\textwidth}{XX XX}
    PSNR: 8.6 & 24.5 & 24.4 & 26.2 \\ 
    SSIM: 0.50 & 0.74 & 0.70 & 0.79 \\ 
    \hline
  \end{tabularx}\\[3mm]
  \begin{tabularx}{\tabwidth\textwidth}{XX XX}
    $\gamma = 0.6$
  \end{tabularx}\\
  \includegraphics[width=\figwidth\textwidth]{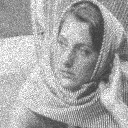}
  \includegraphics[width=\figwidth\textwidth]{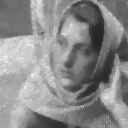}
  \includegraphics[width=\figwidth\textwidth]{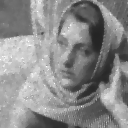}
  \includegraphics[width=\figwidth\textwidth]{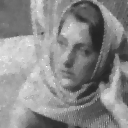} \\
  \begin{tabularx}{\tabwidth\textwidth}{XX XX}
    PSNR: 15.7 & 24.4 & 24.9 & 25.3 \\ 
    SSIM: 0.56 & 0.73 & 0.75 & 0.77 \\ 
    \hline
  \end{tabularx}\\[3mm]
  \begin{tabularx}{\tabwidth\textwidth}{XX XX}
    $\gamma = 0.4$
  \end{tabularx}\\
  \includegraphics[width=\figwidth\textwidth]{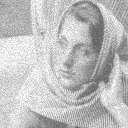}
  \includegraphics[width=\figwidth\textwidth]{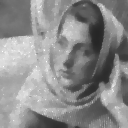}
  \includegraphics[width=\figwidth\textwidth]{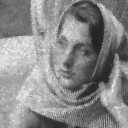}
  \includegraphics[width=\figwidth\textwidth]{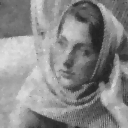} \\
  \begin{tabularx}{\tabwidth\textwidth}{XX XX}
    PSNR: 11.9 & 24.4 & 25.1 & 25.8 \\ 
    SSIM: 0.56 & 0.74 & 0.75 & 0.79 \\ 
  \end{tabularx}\\
  \caption{Examples of combined gamma correction (GC) and denoising compared to ``denoise and then gamma correct" and ``gamma correct and then denoise". If the degradation parameter is set to $\gamma = 1$,  \eqref{eq:gc} is a one-to-one mapping and \eqref{eq:gc_deriv} is the identity. Instead of using the true gamma correction parameter $1/\gamma$ we use $1/(1.1\gamma)$ to represent the uncertainty in the estimate of an unknown inverse mapping. With a stronger nonlinearity, it can be seen that the mapping-based filter performs better than the other approaches, both visually and with respect to error measures.}
  \label{fig:gc}
\end{figure*}

\begin{figure*}
  \centering
  {\includegraphics[width=0.24\textwidth]{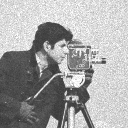}}    
  {\includegraphics[width=0.24\textwidth]{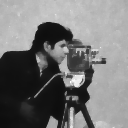}} 
  \hfill\hfill
  {\includegraphics[width=0.24\textwidth]{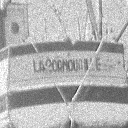}} 
  {\includegraphics[width=0.24\textwidth]{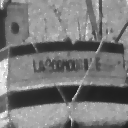}} 
  \hfill\hfill
  \caption{Denoising results for the mapping-based filter when $\gamma = 0.4$. In each instance we accurately preserve the dominant features while suppressing the main parts of the noise. The noise is multiplicative and $\sigma$ was set to 40 levels of the intensity range in the noise model \eqref{eq:multNoise}. }
  \label{fig:gc_examples}
\end{figure*}  

\begin{figure*}
  \centering
  \begin{minipage}{0.26\textwidth}
    \vspace{-10mm}
    \caption{Evaluation of gamma correction of overexposed images Cameraman, Barbara, Lena and Boat. As $\gamma$ becomes smaller, \ie, the nonlinearity is more pronounced, the mapping-based filter consistently performs better than the two alternatives ``denoise and gamma correct" and ``gamma correct and denoise". We set $\sigma = 40$ of the intensity range in \eqref{eq:multNoise}.}
    \label{fig:gc_eval}
  \end{minipage}
  \begin{minipage}{0.73\textwidth}
    \begin{tabular}{cc}
      \imagetop{\includegraphics[width=0.455\textwidth]{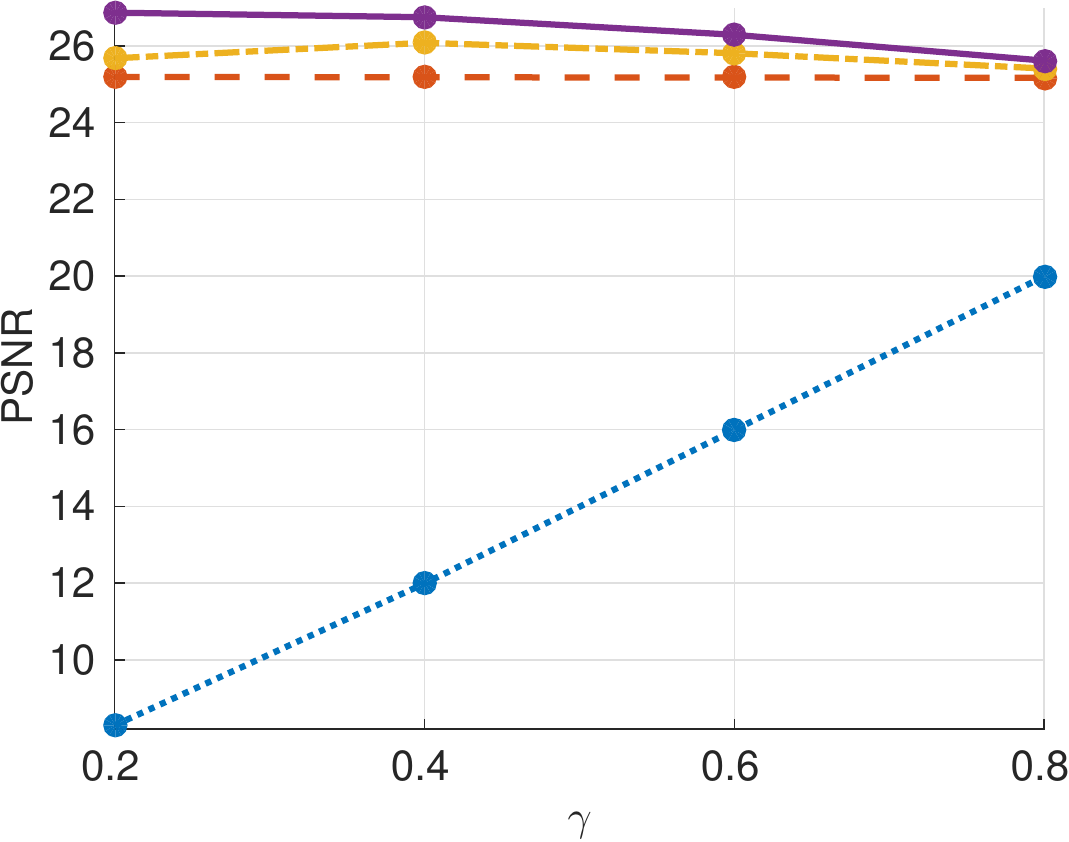}} &   
      \imagetop{\includegraphics[width=0.47\textwidth]{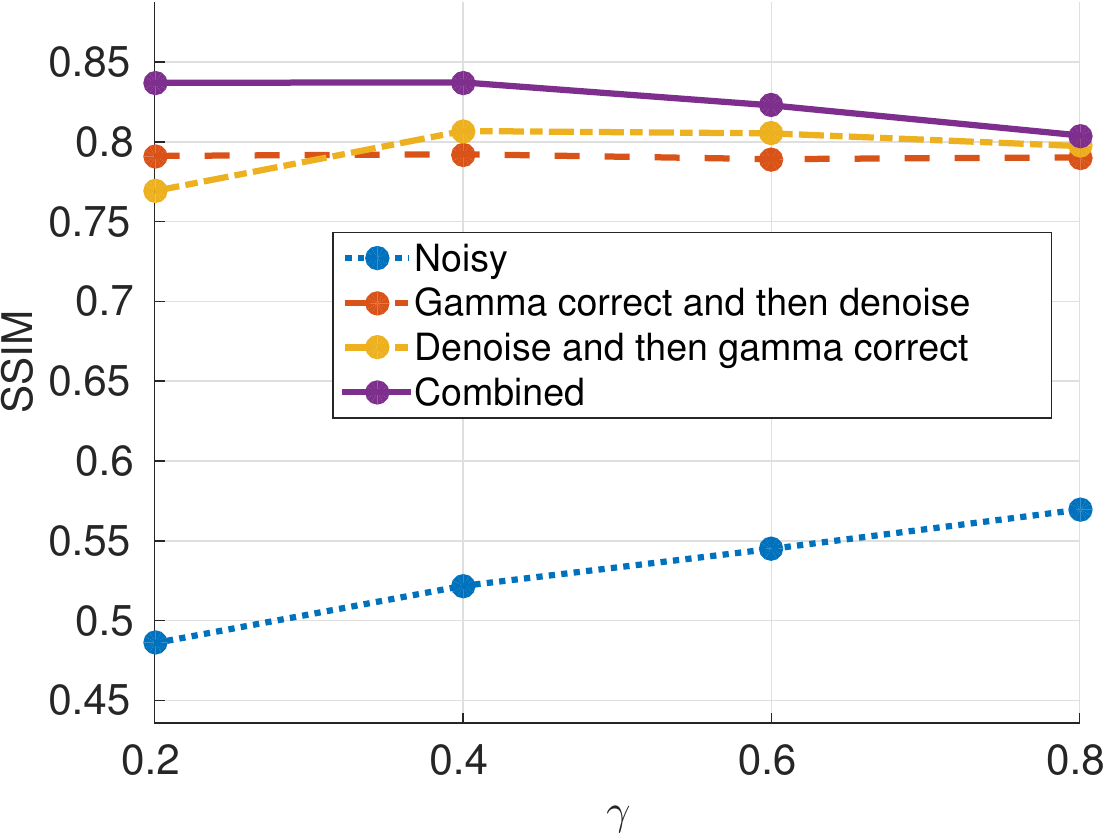}} 
    \end{tabular}
  \end{minipage}
\end{figure*}

In $M$, \eqref{eq:pdeM}, we are required to compute terms such as $1/|\nabla u|$. Therefore, in practice we regularize the denominator term with a small positive constant, \ie, $|\nabla u|_\beta = \sqrt{u_x^2 + u_y^2 + \beta}$ as they appear in $M$. In the gradient energy tensor \eqref{eq:GET}, we are required to compute third-order derivatives (see appendix \eqref{eq:GETa}-\eqref{eq:GETc}) for terms such as $\partial_x \Delta u$. However, rather than to expand these terms using derivative rules, we found that it is sufficient to approximate these terms using a central difference derivative. The finite differences are regularized with a smooth kernel. Here we use a Gaussian filter of standard deviation $1$ before computing the finite differences. The filter size was kept constant for all images and all noise levels in the experimental evaluation. We also fixed the parameters $\lambda$, $\beta$, $\gamma$, $\tau$ and the number of inner iterations $k$ for all experiments. The parameter values are shown in Table \ref{table:parameters}. 

\section{Application I: Gamma correction}	
\label{sec:gc}
In this section we apply total variation with a mapping function (TVm), see \eqref{eq:TVm}, to the problem of denoising overexposed images corrupted with multiplicative noise. Multiplicative noise is the common noise model in, e.g., ultrasound imaging, synthetic aperture radar (SAR) imaging, sonar and laser imaging, etc.

The intention with this application is not to show state of the art results denoising of multiplicative noise, rather we show empirically that (I) gamma correction (GC) and then denoise, or (II) denoising and then gamma correction is suboptimal to (III) using a combined filtering scheme compensating for the gamma correction in the evolution process of the minimiser. The parameters were set as shown in Table \ref{table:parameters}.
We consider the following noise model: let $u^0$ be the degraded image, then 
\begin{equation}\label{eq:multNoise}
  u^0 = \exp(\ln(u) + \eta(u)), 
\end{equation}
where $\eta \sim \mathcal{N}(u | \mu,\sigma)$ and $\mu$ is the mean-value and $\sigma$ the standard deviation. The main difference to Gaussian noise is that multiplicative noise scales with the intensity; the lower the intensity, the less noise is present. Gaussian noise is uniform in the sense that the intensity range is degraded by the same strength of the noise-components. 

We use the ``standard" gamma correction formula and define the mapping function to correct for an overexposed image. In this sense, we define the mapping function that corrects for this degradation as
\begin{equation}\label{eq:gc}
  m(u) = u^{1/\gamma}, 
\end{equation}
where $\gamma>0$. If $\gamma < 1$ then \eqref{eq:gc} corrects for an overexposed image, and if $\gamma > 1$ corrects for an underexposed image. Figure \ref{fig:gamma_mapping} shows the different values for $\gamma$ and its inverse, we also show the difference to 10\% offset of the inverse gamma value. The reason why we include this scale-factor is that in practical applications it is difficult, or impossible, to obtain an accurate estimate of the true inverse-mapping. In the E-L equation we are required to compute the derivative of the mapping, \ie, we have
\begin{equation}\label{eq:gc_deriv}
  m'(u) = \frac{1}{\gamma} u^{1/\gamma - 1},
\end{equation}
and normalize it by the maximum slope so that $0\leq m'(u)\leq 1$ to assure stability when solving the PDE~\eqref{eq:TVm}. 

\begin{figure*}
  \centering
  \includegraphics[width=0.24\textwidth]{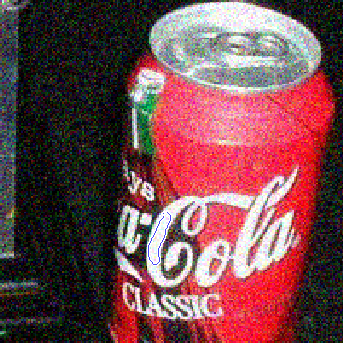}    
  \includegraphics[width=0.24\textwidth]{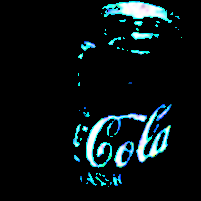} 
  \begin{overpic}[scale=0.59]{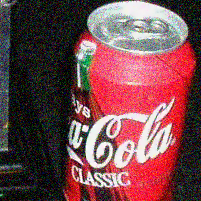}
    \put(3,35){\includegraphics[scale=0.8]{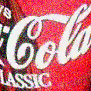}}  
  \end{overpic}
  \includegraphics[width=0.24\textwidth]{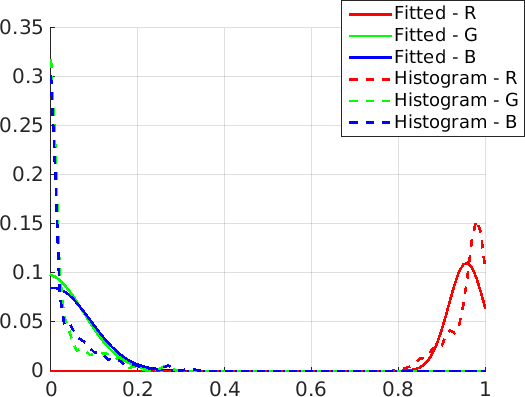}   \\
  \includegraphics[width=0.24\textwidth]{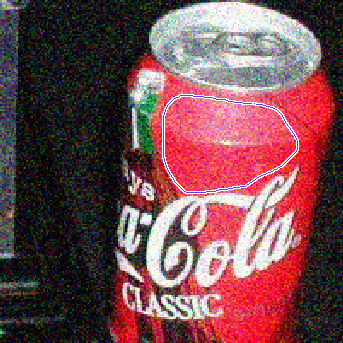}    
  \includegraphics[width=0.24\textwidth]{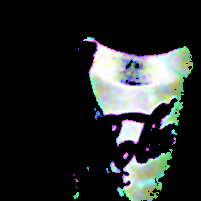} 
  \begin{overpic}[scale=0.59]{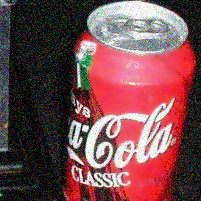}
    \put(3,35){\includegraphics[scale=0.8]{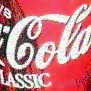}}  
  \end{overpic}
  \includegraphics[width=0.24\textwidth]{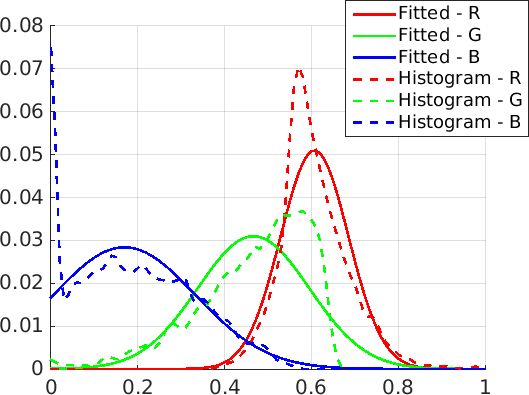}   \\
  \hfill Selected region \hfill \hfill Mapping \hfill\hfill Result \hfill\hfill Histogram \hfill \phantom{I} \\
  \caption{Two examples of targeted value range filtering from user selected regions. On the first row we focus on the text. As visible in the close-up in the Result-column the noise is removed in the white text whereas and preserved in the red background. On the second row we instead choose to filter the red color of the soda can, and as visible in the Result-column the red is denoised whereas the noise remains in the white text. The column labeled ``Mapping'' is computed from the fitted (normalized to 1) density functions in the right-most column. See text for details.} 
  \label{fig:targetedFiltering-Value}
\end{figure*}

\subsection{Results}
Figure \ref{fig:gc} shows the result of denoising an overexposed image for the different cases of filtering order. Naturally, if $\gamma = 1$, then the case of ``filter and then gamma correct" is equivalent to the mapping-based filter (as it should be). As $\gamma$ decreases, \ie, the nonlinearity become more severe, we observe that the mapping-based filter becomes superior to the two other approaches in all cases. The error measures in Figure \ref{fig:gc_eval} clearly illustrates this relation. In particular, Figure \ref{fig:gc_eval} shows the error measures for varying $\gamma$-parameter and the different orders (I), (II), and (III) of filtering. The final error values were obtained at the peak structural similarity value (SSIM) \cite{Wang2004}. The reason for the improved performance is easiest explained in terms of a dynamic system: by including the nonlinear mapping of the intensity range as a component in the filtering scheme, it acts as a feedback component adjusting the diffusivity to conform with the final visualized result. This behaviour is also illustrated in Figure \ref{fig:overview}, Section \ref{intro}. 

In addition to the error graph, we show the final result for the mapping-based filter of Cameraman and Boat as $\gamma = 0.4$ in Figure \ref{fig:gc_examples}. It is visible that most of the noise has been removed and that edges remain well-preserved.

\section{Application II: Targeted value range filtering}
\label{sec:tf}

In this section we present an application where the image value range is set from prior information based on user interaction. 

Figure \ref{fig:example_synthetic} in the introduction illustrates a fixed selection of the image value range where we know which value range is noisy prior to filtering. Thus, we can achieve high quality result where we preserve fine details in the image. These results are not possible to obtain with state of the art denoising methods such as MRF \cite{5539844}, TGV \cite{Bredies:2010:TGV:1958729.1958739} or BM3D \cite{DBLP:journals/tip/DabovFKE07} since they denoise the complete image value range.  

\subsection{User-based value-mapping}
In the introduction we also gave an initial example of user-based region selected value range filtering, see the first row of Figure \ref{fig:example_synthetic_butterfly} which shows a result when filtering the butterfly's wings. In the same figure, the corresponding mapping shows the region of interest that is processed by our algorithm. The denoised result is shown in the right column. It is visible that the noise in the desired region is reduced meanwhile noise in other value ranges are preserved.  

We give an additional example in Figure \ref{fig:targetedFiltering-Value}. In the right most column we also show the histogram for the selected color region in the left column. The histogram shows the sample value-distribution of each color channel and the corresponding estimated Gaussian distribution. These distributions make up our channel specific mapping functions. In this targeted-value range application we first reduce the noise of the letters written on the soda can. This is achieved by selecting a sample region in the image outlined on the letter ``C'', displayed on the first column in the first row. The resulting mapping shows that the algorithm is targeting the desired region. A close-up in the result-column shows that the filtering is confined to the letters - leaving the noise in the remainder of the value range undisturbed. 

Next, we mark a region with red color and filter the image. The corresponding close-up on the second row shows that the red color is free of noise and that the white color of the letters still contains noise, just as desired. This type of filtering process is beneficial when the application's noise-component is known prior to filtering. Also, there may be cases when one can tolerate the preservation of noise meanwhile other regions are smoothed. 

When performing the filtering process the fitted density functions are normalized in the same way as was done in the application of gamma correction in the previous section.

\section{Application III: Denoising via an oversegmentation map}
\label{sec:colour_eval}

In the third application we treat the problem of global image denoising, \ie, in contrast to applications I and II, we now consider the whole (spatial) image domain $\Omega$. Here, the mapping function will be estimated from locally homogeneous regions via an oversegmentation map.

\subsection{Method definitions}
We define the tensor $W$ and the mapping in functional \eqref{eq:Rfunctional} to describe isotropic diffusion (ISO), isotropic diffusion with a mapping function (ISOm) \cite{Astrom2013a},  total variation (TV) \cite{Rudin1992}, total variation with a mapping function (TVm), see \eqref{eq:TVm}, gradient energy tensor total variation (GETV) \cite{Astrom2014} and GETV with a mapping function (GETVm), see \eqref{eq:E-L_3terms}. The definitions of these methods are outlined in Table \ref{tab:overview}. The  diffusivity parameter in \eqref{eq:Dget+} for GETV and GETVm is set according to
\begin{equation}
  \label{eq:k2}
  k^2 = \frac{\exp(1)-1}{\exp(1)-2}\sigma^2,
\end{equation}
where $\sigma^2$ is the image noise variance \cite{Fberg2011}. In the following evaluation the aim is not to show superiority over all existing denoising algorithms. We rather investigate how an oversegmentation map affects the final denoised result and we believe interesting results will follow from this approach. To get comparative performance result of the below outlined approach, we compare our formulation to original code of state of the art denoising methods BM3D \cite{DBLP:journals/tip/DabovFKE07}, TGV \cite{Bredies:2010:TGV:1958729.1958739} and MRF \cite{5539844}.

\subsection{Approach and motivation}
Objects in an image, or more specifically, their value distributions, can be described by a stochastic process. Such processes can be estimated from image segments, or homogeneous regions, via an oversegmentation map. Figure \ref{fig:oversegm_boat_map} shows one such partitioning of similar regions. In this case, each region (or object) is associated with its own set of features \cite{citeulike:6171513}. In image denoising the aim is to remove the noise while preserving image features. In other words we want to obtain a ``clean" image without the noise component, \ie, we want to infer the true underlying statistical process that describes the noise free image. Therefore, modelling the image filtering process based on estimated densities is a means to incorporate a prior into the filtering scheme. Here, the aim is to separate objects based on certainty estimates originating from the underlying distributions. In practice we rely on the principle of ergodicity to estimate first and second order moments of each segments sample distribution. Additionally, we assume uncorrelated pixels between segments to drive the diffusion process. The idea is to exploit properties of statistical sample distributions in a way such that we can reduce the uncertainty to which segment a pixel, located at the border of two segments, belongs to. This is achieved by incorporating the statistical properties of pixel density functions into the filtering scheme.

Based on this motivation we let the derivative of the mapping function, \ie, $m'$, describe the sample distribution. Thus, the mapping function can be interpreted as the corresponding cumulative density function. 

In this application, our approach to image denoising is a four step procedure: (I) Compute an oversegmentation that conforms to homogeneous regions. (II) In the case of natural colour images, transform the colour space into colour opponent components. (III) Construct the mapping derivative, $m'$, from estimated density functions. (IV) Solve the corresponding E-L equation for a given $W$. We use the CIELAB colour representation of the noisy images to decorrelate its RGB components \cite{sharma2014digital}. In the following we introduce the concept of oversegmentation maps and show how they can be used to compute mapping functions to drive the diffusion process. 

\subsection{Oversegmentation}
\label{sec:oversegmentation}
The introduction of a mapping into the diffusion framework was first presented in \cite{Astrom2013a}. In this work we extend the initial formulation to include a tensor component, this is enabled by the functional in Theorem~\ref{thm:R}. However, before specifying the filtering method, we explain how to obtain an oversegmentation map to estimate the image statistics. 

\renewcommand{\figwidth}{0.49}
\begin{figure*}[t]
  \newcommand{\tabwidth}{1}	
  \centering
  \begin{minipage}{0.33\linewidth}
    \vspace{-105mm}
    \caption{Image ``Boat''. Top left original image, right is the degraded image corrupted with additive Gaussian noise with standard deviation 20. Down left, oversegmentation map and down right, is the obtained mapping function.}
    \label{fig:oversegm_boat_map}
  \end{minipage}
  \begin{minipage}{0.66\linewidth}
    \hfill  \includegraphics[width=\figwidth\textwidth]{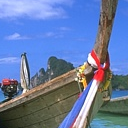}
    \hfill \hfill \includegraphics[width=\figwidth\textwidth]{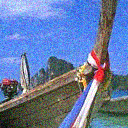} \hfill \phantom{I} \\[-3mm]
    \begin{tabularx}{\tabwidth\textwidth}{XX}
      Original &  Noisy  
    \end{tabularx}\\
    \hfill \includegraphics[width=\figwidth\textwidth]{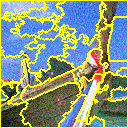} 
    \hfill \hfill \includegraphics[width=\figwidth\textwidth]{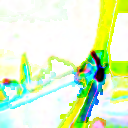} \hfill \phantom{I} \\[-3mm]
    \begin{tabularx}{\tabwidth\textwidth}{XX}
      Oversegmentation &  Mapping \\
    \end{tabularx}\\
  \end{minipage}
\end{figure*}

The selection of image regions suitable for image denoising is a non-trivial task. Much research has focused on devising efficient methods to partition images into homogeneous areas that adhere to image structure boundaries. It is not until recently these partitioning methods are efficient enough to produce sufficiently accurate regions for image denoising while achieving computation times suitable for practical applications. One oversegmentation method that satisfy these demands is SEEDS, superpixels extracted via energy-driven sampling~\cite{SEEDS}. With the assumption that each segment in the oversegmentation map is independent, we estimate local second order statistics for each segment. A sensitivity analysis on the selection of number of regions was done in~\cite{Astrom2013a}, which showed that the filtering process is robust to the number of segments used in the oversegmentation map. Furthermore, note that the selection of oversegmentation method is not unique. We expect that as the oversegmentation methods continue to improve boundary localizations, they may further accelerate their applicability in denoising applications. 

\subsection{Estimating a mapping}
The spatial extent of a segment in the oversegmentaion map should conform to approximately homogeneous regions. If the requirement is fulfilled, then the observed data (or pixels) within one region can be assumed to originate from the same stochastic process and modeled using the same density function. We use this idea to derive a confidence value that describes the certainty that a sample (pixel) belongs to a region or not, \eg, a noisy observation would have a low confidence. We model the mapping $m$ to reflect the confidence level that a sample belongs to a neighbourhood. To do so we define the mapping from a local neighbourhood $\mathcal{L}$. Assuming independence between each segment, we model the data sample in the neighbourhood using a Gaussian distribution with associated mean $\mu$ and variance $\sigma^2$. In practice we estimate the statistical moments using the sample mean and variance of the segment. The mapping function is defined as
\begin{equation} \label{eq:segmentation_m}
  m'(u;\mu,\sigma) = \prod_{i\in \mathcal{L}} g(u; \mu_i, \sigma_i),
\end{equation}
where
\begin{equation}
  g(u; \mu_i, \sigma_i) = \frac{1}{\sqrt{2\pi}\sigma_i} \exp \left ( {-\frac{1}{2}\frac{(u-\mu_i)^2}{\sigma_i^2}} \right ). 
\end{equation}
Thus, instead of controlling the filtering based on edge-strength, \eg, as in P-M \eqref{eq:PMg}, the filtering is now defined by $m'$, the certainty that an observed sample belong to the local neighbourhood. The selection of neighbourhood $\mathcal{L}$ is arbitrary, but for simplicity we choose its size to be $3 \times 3$. In~\eqref{eq:segmentation_m} we use the product of the probability density functions (PDF) rather than the sum which would yield a proper Gaussian mixture model (GMM). The motivation is that at edges, the GMM would not reduce the filtering as efficiently as in the case of a product mapping function. When filtering colour images we compute one oversegmentation of the noisy image. This map is then used to compute the density estimates in each channel of the colour image representation. 

\renewcommand{\figwidth}{0.32}
\begin{figure*}[t]
  \newcommand{\tabwidth}{1}	
  \centering
  \hfill  \includegraphics[width=\figwidth\textwidth]{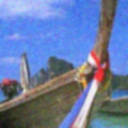} \hfill \hfill
  \hfill\hfill  \includegraphics[width=\figwidth\textwidth]{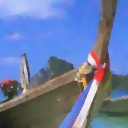} \hfill \hfill
  \hfill\hfill  \includegraphics[width=\figwidth\textwidth]{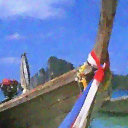} \hfill \phantom{I} \\
  \begin{tabularx}{\tabwidth\textwidth}{XXX}
     ISO &  TV  &  GETV \\
  \end{tabularx}\\
  \hfill  \includegraphics[width=\figwidth\textwidth]{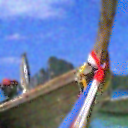} \hfill \hfill
  \hfill\hfill  \includegraphics[width=\figwidth\textwidth]{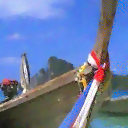} \hfill \hfill
  \hfill\hfill  \includegraphics[width=\figwidth\textwidth]{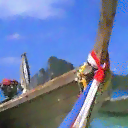} \hfill \phantom{I} \\
  \begin{tabularx}{\tabwidth\textwidth}{XXX}
    ISOm &  TVm &  GETVm \\
  \end{tabularx}\\
  \hfill \includegraphics[width=\figwidth\textwidth]{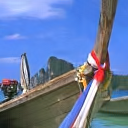} \hfill \hfill
  \hfill \hfill \includegraphics[width=\figwidth\textwidth]{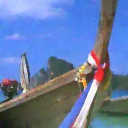} \hfill \hfill
  \hfill \hfill \includegraphics[width=\figwidth\textwidth]{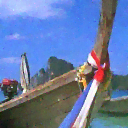} \hfill \phantom{I} \\
  \begin{tabularx}{\tabwidth\textwidth}{XXX}
    BM3D &  TGV  & MRF \\
  \end{tabularx}
  \caption{Comparison of visual quality. The Boat's bow appears more clear where we have applied the mapping function, compared to ISO, TV and GETV. Visually, TGV appear similar to TVm and GETVm. While ISO and ISOm are oversmoothed, MRF and GETV exhibits color artifacts in the blue sky. Although, BM3D has good visual and good error values (Table \ref{eq:table_boat}), its approach to denoising is \emph{fundamentally} different from the PDE-based method introduced in this work.}
  \label{fig:oversegm_boat}
\end{figure*}

\begin{table*}
  \centering
  \begin{minipage}{0.48\linewidth}
    \centering
    \renewcommand*{\arraystretch}{1.3}
    \begin{tabularx}{0.55\textwidth}{r|c|c}
      & \textbf{PSNR} & \textbf{SSIM} \\
      \hline
      Noisy  & 22.10 & 0.47 \\ 
      \hline 
      ISO  & 25.09 & 0.75 \\ 
      ISOm & 25.36 & 0.76 \\ 
      \hline 
      TV   & 26.71 & 0.80 \\ 
      TVm  & 26.66 & 0.80 \\ 
      \hline  
      GETV & 27.60 & 0.77 \\ 
      GETVm & 26.70 & 0.80 \\ 
      \hline  
      TGV  & 27.98 & 0.82 \\ 
      BM3D & 31.36 & 0.90 \\ 
      MRF  & 28.26 & 0.81 \\ 
    \end{tabularx} 
    \caption{Error values for the Boat image in Figure \ref{fig:oversegm_boat_map} and \ref{fig:oversegm_boat}. We see that the mapping function in particular improves the SSIM values. Compared to TGV, BM3D and MRF our results are most similar to TGV and MRF.}
    \label{eq:table_boat}
  \end{minipage}
  \hspace{3mm}
  \begin{minipage}{0.48\linewidth}
    \centering
    \renewcommand*{\arraystretch}{1.3}
    \begin{tabularx}{0.55\textwidth}{r|c|c}
      & \textbf{PSNR} & \textbf{SSIM} \\
      \hline
      Noisy  & 22.08 & 0.40 \\ 
      \hline 
      ISO  & 27.69 & 0.80 \\ 
      ISOm & 28.09 & 0.81 \\ 
      \hline 
      TV   & 29.05 & 0.83 \\ 
      TVm  & 29.32 & 0.83 \\ 
      \hline  
      GETV & 29.12 & 0.76 \\ 
      GETVm & 29.42 & 0.84 \\ 
      \hline  
      TGV  & 30.03 & 0.85 \\ 
      BM3D & 33.10 & 0.91 \\ 
      MRF  & 29.71 & 0.82 \\ 
    \end{tabularx} 
    \caption{Error values for the Lena image in Figure \ref{fig:oversegm_lena_map} and \ref{fig:oversegm_lena}. In this case the mapping improves all error values. Compared to TGV, BM3D and MRF our results are most similar to TGV and MRF.}
    \label{eq:table_lena}
  \end{minipage}
\end{table*}

\renewcommand{\figwidth}{0.49}
\begin{figure*}[t]
  \newcommand{\tabwidth}{1}	
  \centering
  \begin{minipage}{0.33\linewidth}
    \vspace{-103mm}
      \caption{Image ``Lena''. Top left original image, top right shows the degraded image corrupted with additive Gaussian noise of standard deviations 20. Down left shows the oversegmentation map of the noisy image and down right illustrates the obtained mapping function.}
      \label{fig:oversegm_lena_map}
  \end{minipage}
  \begin{minipage}{0.66\linewidth}
    \hfill  \includegraphics[width=\figwidth\textwidth]{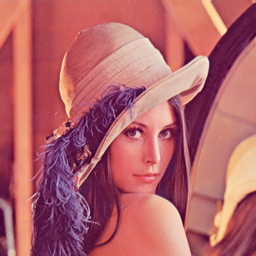} 
    \hfill \hfill \includegraphics[width=\figwidth\textwidth]{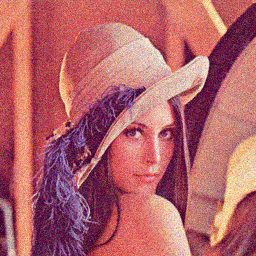} \hfill\phantom{I} \\[-3mm]
    \begin{tabularx}{\tabwidth\textwidth}{XX}
      Original  & Noisy \\
    \end{tabularx}\\
    \hfill \includegraphics[width=\figwidth\textwidth]{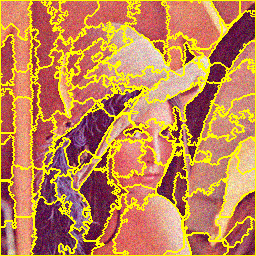} 
    \hfill \hfill \includegraphics[width=\figwidth\textwidth]{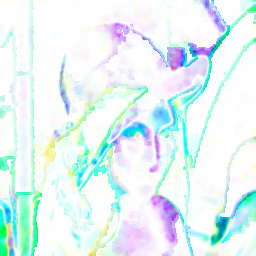} \hfill \phantom{I} \\[-3mm]
    \begin{tabularx}{\tabwidth\textwidth}{XX}
      Oversegmentation  & Mapping \\
    \end{tabularx}\\
  \end{minipage}
\end{figure*}

\renewcommand{\figwidth}{0.32}
\begin{figure*}[t]
  \newcommand{\tabwidth}{1}	
  \centering
  \hfill  \includegraphics[width=\figwidth\textwidth]{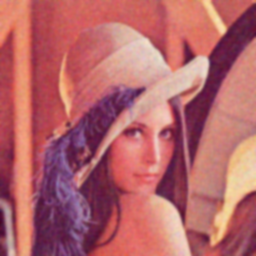} \hfill \hfill
  \hfill\hfill  \includegraphics[width=\figwidth\textwidth]{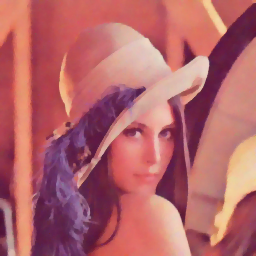} \hfill \hfill
  \hfill\hfill  \includegraphics[width=\figwidth\textwidth]{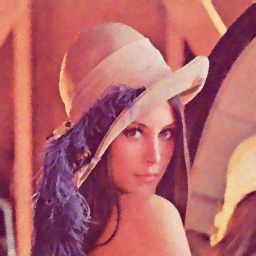} \hfill \phantom{I} \\
  \begin{tabularx}{\tabwidth\textwidth}{XXX}
     ISO &  TV  &  GETV \\
  \end{tabularx}\\
  \hfill  \includegraphics[width=\figwidth\textwidth]{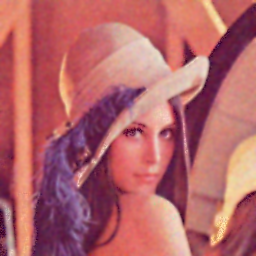} \hfill \hfill
  \hfill\hfill  \includegraphics[width=\figwidth\textwidth]{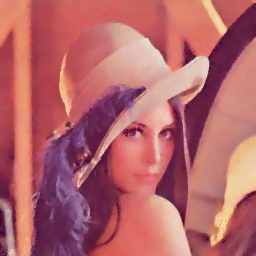} \hfill \hfill
  \hfill\hfill  \includegraphics[width=\figwidth\textwidth]{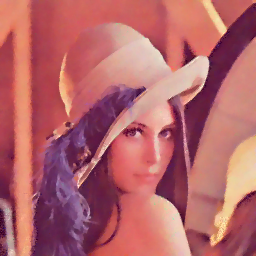} \hfill \phantom{I} \\
  \begin{tabularx}{\tabwidth\textwidth}{XXX}
    ISOm &  TVm &  GETVm \\
  \end{tabularx}\\
  \hfill \includegraphics[width=\figwidth\textwidth]{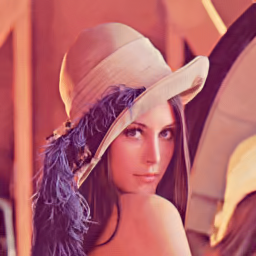} \hfill \hfill
  \hfill \hfill \includegraphics[width=\figwidth\textwidth]{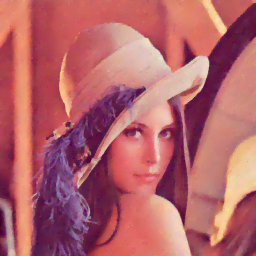} \hfill \hfill
  \hfill \hfill \includegraphics[width=\figwidth\textwidth]{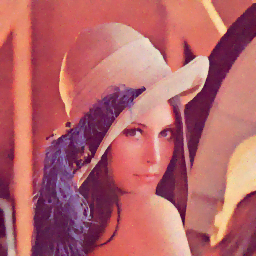} \hfill \phantom{I} \\
  \begin{tabularx}{\tabwidth\textwidth}{XXX}
    BM3D &  TGV  & MRF \\
  \end{tabularx}
  \caption{Comparison of visual quality. The brim of Lena's hat appears more clear where we have applied the mapping function, compared to ISO, TV and GETV. As expected, ISO and ISOm are oversmoothed. TVm, GETVm and TGV all have similar visual appearance. In this example image, there are no apparent color artifacts. Visually, Lenas hair retains more details in BM3D, TGV and MRF compared with TVm and GETVm.}
  \label{fig:oversegm_lena}
\end{figure*}

\subsection{Results}
We already mentioned in the introductory paragraph of this section that we do not aim to outperform state of the art denoising methods. Rather we study how the usage of an oversegmentation map influence the noise removal. Our PDE-framework defines a natural way of reducing the filtering process at border regions rich with detail and texture, so as to not oversmooth corners and lines. We show that the mapping function is the key component for this approach. 

Figure \ref{fig:oversegm_boat} illustrates a result for all considered methods, obtained at their maximum SSIM values. The right illustration in Figure \ref{fig:oversegm_boat_map} shows the mapping for GETVm. We see that the red color-region on the bow of the Boat has high uncertainty, thus in this region, the mapping-based algorithms will reduce the filtering. Note that the color coding is in the RGB-space in the mapping-figure although the actual filtering is done in the CIELAB-space. This means that red color in the mapping corresponds to filtering in the CIELAB intensity channel whereas green and blue describe the opponent color channels. White color indicates that each CIELAB-value is well described by the estimated sample distributions and therefore we have high filtering in each CIELAB channel. Moreover, the color cyan indicates filtering of only the chroma information, not the structural information present in the intensity channel. Due to the mapping, the filtering is reduced in regions with high uncertainty. Example of such regions are the sail and the mountain in the background, see, \eg, ISOm in Figure \ref{fig:oversegm_boat}. Here, this gives an improved visual appearance without oversmoothing details in these regions while suppressing the noise in the background. With respect to the error measures for the Boat image, see Table \ref{eq:table_boat}, it is obvious that BM3D outperforms all compared methods. However, GETV is on-par with TGV and MRF. The mapping makes the most difference for GETV, showing more improvement than any of the other approaches. Although the PSNR value decreases marginally for TVm compared to TV, the denoised result in Figure \ref{fig:oversegm_boat} appears less oversmoothed for TVm. 

Figure \ref{fig:oversegm_lena_map} shows the original and the noisy image of Lena. The same figure also depicts the mapping produced by the GETVm method. Figure \ref{fig:oversegm_lena} illustrates the result of filtering the Lena image for the respective method. In this case, image regions are well-preserved by the mapping, \eg, the brim of the hat and the hair. Note that the mapping-based methods have a tendency to oversmooth features in highly textured regions, \eg, the feathers in Lena's hat. This is due to the violation of homogeneous partitioning of the oversegmentation map, clearly visible in Figures \ref{fig:oversegm_boat_map} and \ref{fig:oversegm_lena_map}. This is a frequent problem for these types of image structure. With respect to the error measures, Table \ref{eq:table_lena}, we see that BM3D is again superior to all considered methods. We also note that the introduction of mapping function does improve the ISO, TV and GETV error values and visual appearance. 

\section{Discussion and summary of results}
\label{sec:discussion}
This work introduces an alternative viewpoint to image filtering, as such, we discuss the usability of the method and give some additional interpretations. 

The novelty of our approach is the idea of including a mapping of the image value range in the variational framework. To this end, the resulting PDE does not only describe spatially-adaptive filtering but also depends on the image value range. Important to note is that the mapping $m$ can be defined as either a \emph{local} or a \emph{global} function. By this we mean that the mapping can be defined locally, \ie, in a region, or alternatively, globally as in the application of gamma correction where the mapping is applied pixel-wise. We shortly elaborate on the difference between these classes of mapping functions. 

For interesting image processing applications, \eg, the ones considered in this work, it is formally necessary to decouple the derived Euler-Lagrange equation from the initial variational formulation. The reason is that none of the complex definitions of the mapping and the nonlinear mappings of the tensor GET$^+$ allows us obtain an energy functional on the form \eqref{eq:EL-div}. This is due to the result in Proposition \ref{prop:GF_conv}. However, note that the decoupling is only in the case GET has negative eigenvalues. For an extended discussion of GET and its properties see \cite{diva2:269100}. If we instead consider the structure tensor: the energy and the E-L equation are \emph{fully} decoupled due to the convolution operator, we again refer to Proposition \ref{prop:GF_conv}.

The versatility of our framework allows a large number of image processing problems to be considered. In this work we illustrate applications of image denoising. We hypothesize that the mapping can be used in, \eg, image value range compression, visualization, deblurring and inpainting, etc. We also see possible extensions to estimate (or learn) the mapping from datasets of images or self-similarity measures of patches in an image.

Naturally our approach is not the only way to achieve any of the considered applications. However, it is a reasonable assumption that coupling the filtering process with the nonlinear mapping is a preferred formulation. Our method is a framework that provides a natural and explicit coupling for image enhancement: it relates theoretical expositions with user-defined as well as data-specific priors. This versatility, comes at a cost of performance. In fact, one can tailor methods to solve each application independently, this is what BM3D has done with great success for the application of image denoising. Although heuristic methods can give better performance w.r.t. error measures, their scalability to other applications remain limited.

To further motivate our approach and clarify its potential we wish to highlight specific differences between the mapping functions we have considered in the respective applications. In particular we discuss, what we call, \emph{local} and \emph{global} mapping functions and in what context the mapping is a suitable approach. 

Table \ref{table:overview_app}, Section \ref{intro}, gives an overview of the different mapping functions considered in this work. 

\subsection{Global mappings}
A global mapping is a mapping where each image sample (pixel) is mapped with the same definition of the mapping function, \ie, $u \mapsto m(u)$. 

This type of mapping is used in the first application on gamma correction. We emphasize that in this case the mapping does not provide any additional \emph{explicit} edge information of the image. However, the mapping does provide information how the image values \emph{should} map, \eg, to correct for the gamma distortion in the final result. Recall that the property of the mapping is that if a line/corner, or feature, is present in the image value range after the mapping, then the feature is preserved. If the structure is not present in the value range after the mapping then the feature is not preserved. Again, although the purely global mapping formulation does not give an explicit edge-stopping effect, the \emph{mapping will aid the discontinuity preservation} induced by the PDE given that the \emph{feature information is present after the mapping}. Moreover, since the mapping is nonlinear, a weak edge in the input data may be transformed to a strong edge as the mapping is applied. Thus, the mapping gives an \emph{implicit} edge-preservation effect. Next we contrast the global mapping function with the local mapping. 

\subsection{Local mapping}
We have considered two types of local mappings. First we defined a local mapping of the color image value range (Application II), where we target the filtering process to a certain color range. The selection of the color range is based on user-input but may also be generalized to prior value ranges based on application requirements. The advantage of having a coupled value and spatial-range filtering process is clear in this targeted application: one does not need to explicitly handle the transition (border) between filtered and unfiltered value range data. This is implicitly handled by the smoothness of the mapping function. 

In the third application, we defined an oversegmentation map where each segment in the map describe a local homogeneous region. From these segments we estimate local density functions and via a mixture model we achieve an \emph{edge stopping effect} between the segments. This edge-stopping effect can be seen in Figure \ref{fig:example_synthetic_butterfly} (second row, middle figure) and in the mappings in Figure \ref{fig:oversegm_boat_map} and \ref{fig:oversegm_lena_map}, respectively. Individually, density functions do not provide any information about edge-localization, rather it is the \emph{combination} of densities that prove edge-localization. This combination of mappings depends on a large number of pixels and their relation. This is comparable, but goes beyond the local estimation of the edge-stopping function in P-M diffusion etc. The method used to produce the oversegmentation map is not unique. Moreover, one can apply any method that produces an oversegmentation map. Naturally, the effect of the mapping in the filtering process will depend on the accuracy of detected segment boundaries. 

\section{Conclusions}
\label{sec:conclusion}
This work introduced a mapping of the image value range into a variational framework. The key component is the mapping function that allows us to replace the ad-hoc edge-stopping function otherwise present in diffusion schemes. After adopting the mapping function to a set of image processing problems we conclude that the mapping contributes positively to the filtering process. Moreover, we have shown that the mapping can be estimated (or set) for each application in a non-parametric way. Our main finding is that context aware tensor-based image filtering can be derived from principles of energy minimization. However, with respect to general image denoising, state of the art denoising methods outperform our framework. If we consider targeted value range applications our methodology clearly demonstrates advantageous aspects by improving feature preservation while reducing the image noise. It is expected that, under appropriate mappings, related energy based problem will benefit from the presented work. Naturally, assuming a reasonable map can be formulated, the direct extension to these alternative applications is to generalize existing regularization terms. 

\begin{acknowledgements}
  We thank the reviewers for their helpful comments and suggestions which have improved this work. This research has received funding from the Swedish Foundation for Strategic Research through the grant VPS and from Swedish Research Council through grants for the projects energy models for computational cameras (EMC$^2$) and Visualization adaptive Iterative Denoising of Images (VIDI), all within the Linnaeus environment CADICS and the excellence network ELLIIT. Support by the German Science Foundation and the Research Training Group (GRK 1653) is gratefully acknowledged by the first author.
\end{acknowledgements}

\appendix

\section{Proof of Theorem~\ref{thm:R}}
\label{appendix:proof_quad_func}
To prove Theorem~\ref{thm:R}, compute the variational derivative of the functional~\eqref{eq:Rfunctional}. The first variation is given by
\begin{equation}
  \delta R = \left . \frac{\partial}{\partial \varepsilon} R(u+\varepsilon v) \right |_{\varepsilon = 0}.
\end{equation}
Now, we let $u \mapsto u+\varepsilon v$ in $R(u)$ and we get
\begin{subequations}
  \begin{align}
    \notag  R(u+\varepsilon v) & = \nabla m(u+\varepsilon v)^{\top}W(\nabla m(u+\varepsilon v)) \nabla m(u+\varepsilon v)  \\
    \label{eq:R_prod_A} & = \underbrace{m'(u+\varepsilon v)^2\nabla (u+\varepsilon v)^{\top}}_{=A} \\ 
    \label{eq:R_prod_B} & \qquad  \cdot  \underbrace{W(\nabla m(u+\varepsilon v)) \nabla(u+\varepsilon v)}_{=B} ,
  \end{align}
\end{subequations}
then, from the product rule of differentiation, we need to consider the following terms
\begin{align}\label{eq:R_prod_exp}
\notag  \frac{\partial}{\partial \varepsilon} R(u+\varepsilon v) &= \left ( \frac{\partial}{\partial \varepsilon} A(u+\varepsilon v) \right)B \\
  & \qquad +  A \left ( \frac{\partial}{\partial \varepsilon} B(u+\varepsilon v) \right).
\end{align}
To simplify the notation define 
\begin{equation}
  z = \begin{pmatrix} z_1 = m'(u+\varepsilon v)(u_1 + \varepsilon v_1) \\ \vdots \\ z_d = m'(u+\varepsilon v)(u_d + \varepsilon v_d)  \end{pmatrix}, 
\end{equation}
and note the relation
\begin{equation}
	z|_{\varepsilon = 0} = \nabla m(u) = s. 
\end{equation}
We first consider the differentiation of the $B$-component w.r.t $\varepsilon$. In order to differentiate the $B$-component use the chain rule of differentiation:
\begin{equation*}
  \pd{W(z)}{\varepsilon} = \pd{W(z)}{z_1}\pd{z_1}{\varepsilon} + \cdots +   \pd{W(z)}{z_d}\pd{z_d}{\varepsilon},
\end{equation*} 
then
\begin{align}\label{eq:dBz}
\notag & \pd{}{ \varepsilon} B(z)  =  \pd{}{ \varepsilon} \left ( W(z)\nabla (u+\varepsilon v) \right ) \\
    & = \left ( \pd{}{\varepsilon} W(z) \right ) \nabla(u+\varepsilon v) +  W(z)\nabla v,
\end{align}
with
\begin{align}\label{eq:dWz}
\notag  \pd{}{\varepsilon} W(z) &=  \Big [W_{z_1}(z) \Big (m''(u+\varepsilon v)v(u_1+\varepsilon v_1) \\ 
\notag                          & \qquad \qquad + m'(u + \varepsilon v) v_1 \Big)  \\
  \notag  & \hspace*{15mm} \vdots \\
  \notag  & \hspace*{3mm}  + W_{z_d}(z) \Big (m''(u+\varepsilon v)v(u_d+\varepsilon v_d) \\
                         & \qquad \qquad + m'(u + \varepsilon v) v_d \Big ) \Big ] ,
\end{align}
evaluating the limit $\varepsilon \rightarrow 0$ in \eqref{eq:dBz} with \eqref{eq:dWz} we get
\begin{align}\label{eq:dBlim}
\notag \pd{}{ \varepsilon} B(z)\bigg |_{\varepsilon = 0} 
  & = \Big [W_{s_1}(s) \Big (m''(u)v u_1 + m'(u) v_1 \Big)  \\
\notag  & \hspace*{15mm} \vdots \\
\notag  & \hspace*{3mm}  + W_{s_d}(s) \Big (m''(u)v u_d + m'(u) v_d \Big ) \Big ] \nabla u \\
   & \hspace*{5mm} + W(s)\nabla v .
\end{align}
Next, we focus on the second product in~\eqref{eq:R_prod_exp}. Then, with \eqref{eq:dBlim} and after few rewrites, we obtain the following
\begin{subequations}
\begin{align}
\notag & \left . A \left ( \frac{\partial}{\partial \varepsilon} B(u+\varepsilon v) \right) \right |_{\varepsilon = 0} \\
\notag & =  m'(u)^2\nabla u^{\top} \Big [W_{s_1}(s) \Big (m''(u)v u_1 + m'(u) v_1 \Big)  \\
\notag & \hspace*{25mm} \vdots \\
\notag & \hspace*{18mm}  + W_{s_d}(s) \Big (m''(u)v u_d  + m'(u) v_d \Big ) \Big ]\nabla u \\
       & \hspace*{5mm} + m'(u)^2\nabla u^{\top} W(s)\nabla v \\
\notag & =  m'(u)^2\Big [\nabla u^{\top} W_{s_1}(s) \nabla u  \Big (m''(u)v u_1 + m'(u) v_1 \Big)  \\
\notag & \hspace*{25mm} \vdots \\
\notag & \hspace*{18mm}  + \nabla u^{\top} W_{s_d}(s) \nabla u \Big (m''(u)v u_d  + m'(u) v_d \Big ) \Big ] \\
       & \hspace*{5mm} + m'(u)^2\nabla u^{\top} W(s)\nabla v \\
\notag & =  m'(u)^2\Big \langle Q , \Big (m''(u)v \nabla u + m'(u) \nabla v \Big)  \Big \rangle \\
       & \hspace*{5mm} + m'(u)^2 \Big \langle \nabla u,  W(s)\nabla v \Big \rangle  \\
\notag & =  \underbrace{ \Big  \langle m'(u)^2 m''(u)Q, \nabla u\Big \rangle}_{=G} v \\ 
\label{eq:first_derivative}  & \hspace*{5mm} + \Big \langle \nabla v, \Big (\underbrace{ m'(u)^3 Q + m'(u)^2W(s)^{\top}}_{=H} \Big ) \nabla u \Big  \rangle  ,
\end{align}
\end{subequations}
where we set
\begin{align}
  Q = \begin{pmatrix} 
    \nabla u^{\top} W_{s_1}(s) \nabla u \\
    \vdots \\
    \nabla u^{\top} W_{s_d}(s) \nabla u
  \end{pmatrix}
  = \prescript{}{\nabla u}{W_{\bm s}({\bm s})} \nabla u, 
\end{align}
and $ \prescript{}{\nabla u}{W_{\bm s}({\bm s})}$ is defined in \eqref{eq:Ws}. The definitions $G$ and $H$ will be used in the application of Green's formula \eqref{eq:EL-div}. 
Now we focus on the differentiation of the $A$-component w.r.t. $\varepsilon$ in \eqref{eq:R_prod_A}, \ie, we set $u \mapsto u+\varepsilon v$ and get
\begin{align}
\notag & \frac{\partial}{\partial \varepsilon} A(u+\varepsilon v) \\
\notag & =  2m'(u+\varepsilon v)m''(u+\varepsilon v) v \nabla (u+\varepsilon v)^{\top} \\
\label{eq:dA}       & \qquad +  m'(u+\varepsilon v)^2 \nabla v^{\top} .
\end{align}
With \eqref{eq:dA} we evaluate the limit  $\varepsilon \rightarrow 0$ of  \eqref{eq:R_prod_A} which results in
\begin{align}
\notag & \left . \left ( \frac{\partial}{\partial \varepsilon} A(u+\varepsilon v) \right)B \right |_{\varepsilon = 0} \\
\notag  = & \Big ( 2m'(u)m''(u) v \nabla u^{\top} +  m'(u)^2 \nabla v^{\top} \Big ) W(\nabla m(u)) \nabla u \\
\notag  = & \Big (\underbrace{2m'(u)m''(u) \nabla u^{\top}  W(\nabla m(u)) \nabla u}_{E} \Big ) v \\
\label{eq:second_derivative}    & \hspace*{10mm} + \nabla v^{\top} \Big ( \underbrace{m'(u)^2 W(\nabla m(u))}_{F}  \Big ) \nabla u ,
\end{align}
where the definitions $E$ and $F$ will be used in the application of Green's theorem. Summing \eqref{eq:first_derivative} and \eqref{eq:second_derivative} and making use of $G, H$ and $E,F$ yields 
\begin{align}\label{eq:intHF}
\notag  \left . \frac{\partial}{\partial \varepsilon} R(u+\varepsilon v) \right |_{\varepsilon = 0} &= \int_\Omega (G+E)v \; d\bm{x} \\ 
&+ \int_\Omega \nabla v^{\top} (H+F)\nabla u \; d\bm{x}. 
\end{align}
The second integral of \eqref{eq:intHF} is now on the form $\nabla v^\top A(\nabla u)$, \ie, we can integrate it w.r.t. $\nabla v$ by using Green's formula (cmp. \eqref{eq:EL-div}) and get
\begin{align*}
\notag  & \int_\Omega \nabla v^{\top}  (H+F)\nabla u \; {d\bm{x}} \\
& = \int_{\partial \Omega} v (\bm{n}\cdot (H+F)\nabla u) \;dS - \int_{\Omega} v \div{(H+F)\nabla u} \;{d\bm{x}} ,
\end{align*}
where the natural Nuemann boundary condition is given by $\bm{n}\cdot (H+F) = 0$. Substituting $G,H,E,F$ defined in \eqref{eq:first_derivative} and \eqref{eq:second_derivative} into the above expression, and after rearranging the terms and dropping the parentheses for improved clarity, we get
\begin{equation}
  \left \{
    \begin{array}{rlll}
      m'm'' \nabla u^{\top} \Big [ 2W +  \prescript{}{\nabla u}{W_{\bm s}({\bm s})} m' \Big ] \nabla u  \\
	  - \div{ (m')^2 \Big ( W +  W^{\top} + m' \prescript{}{\nabla u}{W_{\bm s}({\bm s})} \Big )\nabla u} & = 0 & \mbox{in} & \Omega \\[3mm]
      n \cdot (F + H)\nabla u & = 0    & \mbox{on} & \partial \Omega
    \end{array} \right .
  \label{eq:ELexp_simplified}
\end{equation}
The above E-L equation can be simplified by considering the following expansion 
\begin{align}\label{eq:proofEL}
\notag  & \div{(m')^2[2W+\prescript{}{\nabla u}{W_{\bm s}({\bm s})} m']\nabla u} \\
\notag  & \hspace{10mm} = 2m'm''\nabla u^{\top} [2W+\prescript{}{\nabla u}{W_{\bm s}({\bm s})} m']\nabla u \\
  & \hspace*{15mm} + (m')^2\div{(2W+\prescript{}{\nabla u}{W_{\bm s}({\bm s})} m')\nabla u} .
\end{align}
Substituting the first term of the right hand side of \eqref{eq:proofEL} into the PDE of \eqref{eq:ELexp_simplified} we get:
\begin{align}
\notag & \div{(m')^2[2W+\prescript{}{\nabla u}{W_{\bm s}({\bm s})} m']\nabla u} \\
\notag & \hspace*{5mm} - (m')^2\div{[2W+\prescript{}{\nabla u}{W_{\bm s}({\bm s})} m']\nabla u}  \\
\label{eq:PDEnoLinearity} & \hspace*{5mm} - 2\div{ (m')^2 \Big [ W +  W^{\top} + m' \prescript{}{\nabla u}{W_{\bm s}({\bm s})} \Big ]\nabla u} = 0 .
\end{align}
The final result is obtained after subtracting the last term from the first term in  \eqref{eq:PDEnoLinearity}, which concludes the proof. \hfill $\qed$

\section{Proof of Corollary \ref{corr:R}}
\label{app:proofGETV}
To derive the gradient energy tensor diffusion scheme we set $W$ and $D$ as specified in \eqref{eq:DW}. With this relation between $W$ and $D$ we introduce $E$ and set $E = 1/|\nabla m(u)|$, then $W = DE$. Before proceeding with the proof, we specify the components of GET. The components of GET \eqref{eq:GET_components} are
\begin{subequations}
\begin{align}
\label{eq:GETa}  a & = (\partial_x u_x)^2 + (\partial_x u_y)^2 - u_x (\partial_{xx}u_x + \partial_{xy}u_y) \\
  \notag	b & = (\partial_{x}u_x)(\partial_x u_y) + (\partial_{y}u_x)(\partial_y u_y) \\ 
\label{eq:GETb}  & \hspace{5mm} - \frac{1}{2}( u_x (\partial_{yx}u_x + \partial_{yy}u_y )  + u_y(\partial_{xx}u_x + \partial_{xy}u_y))  \\
\label{eq:GETc}  c & = (\partial_y u_y)^2 + (\partial_y u_x)^2 - u_y (\partial_{yx}u_x + \partial_{yy}u_y) .
\end{align}
\end{subequations}
Now we need to compute $\prescript{}{\nabla u}{W_{\nabla u}(\nabla u)}$ in \eqref{eq:T_s} with $\bm{s} = \nabla u$, that is, we have
\begin{align}
  \prescript{}{\nabla u}{W_{\nabla u}(\nabla u)} = 
  \begin{pmatrix} 
    \nabla u^{\top} E_{u_x}D \\ 
    \nabla u^{\top} E_{u_y}D
  \end{pmatrix}	
  +	 
  E
  \begin{pmatrix} 
    \nabla u^{\top}  D_{u_x} \\ 
    \nabla u^{\top}  D_{u_y}
  \end{pmatrix} .
\end{align}
The derivatives of $E$ reads
\begin{subequations}
\begin{align}
  E_{u_x} &= \partial_{u_x} \frac{1}{|\nabla u|} = - \frac{1}{|\nabla u|^3} u_x,  \\
  E_{u_y} &= \partial_{u_y} \frac{1}{|\nabla u|} = - \frac{1}{|\nabla u|^3} u_y.
\end{align}
\end{subequations}
The tensor $\prescript{}{\nabla u}{W_{\nabla u}(\nabla u)}$ in the E-L scheme, can now be expressed as~\eqref{eq:corr1_Ts}. \hfill \qed

\section{Eigendecomposition}
\label{app:eigenD}
This part decomposes $D$ in its eigendecomposition and computes $D_{u_x}$ and $D_{u_y}$. We have
\begin{equation}
\notag	D(\nabla u) = U \Lambda U^{\top}  
	 =
	\begin{pmatrix} 
		v^2_1   &   v_1v_2   \\
		v_1v_2  &  v^2_2
	\end{pmatrix}\lambda_1
	+
	\begin{pmatrix} 
		w^2_1   &   w_1w_2   \\
		w_1w_2  &  w^2_2
	\end{pmatrix}\lambda_2,
\end{equation}
where $\lambda_{1,2} = \exp(-|\mu_{1,2}|/k^2)$ and $\mu_1 = \frac{1}{2}( \tr{GET} + \alpha)$ and $\mu_2 = \frac{1}{2}(\tr{GET} - \alpha)$ are the eigenvalues of GET with $\alpha = \sqrt{(a-c)^2 + 4b^2}$. We obtain the eigenvectors of GET by solving $GET \tilde{v} = \mu_1 \tilde{v}$, \ie, we have the following equation system
\begin{equation}
  \left \{
    \begin{array}{cccccc}
      (a - c - \alpha)\tilde{v}_1 & + & 2b\tilde{v}_2 & = & 0 & \\
      2b\tilde{v}_1                    & + & (c-a-\alpha)\tilde{v}_2 & = & 0 & \enspace .
    \end{array} \right . 
\end{equation}
The orthonormal eigenvectors of GET are \eqref{eq:eigenvector_v} and \eqref{eq:eigenvector_w}. Now, we focus on $\partial_{u_x} D$ in \eqref{eq:D_ux}. After expanding the derivatives of the eigenvectors we obtain
\begin{subequations}
\begin{align}
  \partial_{u_x} \tilde{v}_1 &= -(\partial_{yx}u_x + \partial_{yy}u_y) , \\
  \partial_{u_x} \tilde{v}_2 &= \partial_{xx}u_x + \partial_{xy}u_y + \partial_{u_x} \alpha , \\
  \partial_{u_x} \tilde{w}_1 &= \partial_{u_x} \tilde{v}_1,  \\
  \partial_{u_x} \tilde{w}_2 &= \partial_{xx}u_x + \partial_{xy}u_y - \partial_{u_x} \alpha, 
\end{align}
\end{subequations}
and 
\begin{subequations}
\begin{align}
  \partial_{u_x} |\tilde{v}| &= |\tilde{v}|^{-1}(\tilde{v}_1(\partial_{u_x}\tilde{v}_1)+\tilde{v}_2(\partial_{u_x}\tilde{v}_2)), \\
  \partial_{u_x} |\tilde{w}| &= |\tilde{w}|^{-1}(\tilde{w}_1(\partial_{u_x}\tilde{w}_1)+\tilde{w}_2(\partial_{u_x}\tilde{w}_2)).
\end{align}
\end{subequations}
The derivatives of the tensor's eigenvalues are:
\newcommand{\sgn}{\mathop{\mathrm{sgn}}}
\begin{align}
\notag	\partial_{u_x} \lambda_{1,2} & = \partial_{u_x} \exp(-|\mu_{1,2}|/k^2) \\
& = - \frac{\sgn(\mu_{1,2})}{k^2}\Big(\partial_{u_x}\tr{GET} \pm \partial_{u_x}\alpha \Big)\lambda_{1,2}, 
\end{align}
where 
\begin{equation}
  \partial_{u_x}\alpha =  \alpha^{-1} \Big ( \tr{GET}\partial_{u_x}\tr{GET} - 2\partial_{u_x} \det(GET) \Big)
\end{equation}
 and
\begin{subequations}
\begin{align}
  \partial_{u_x}\tr{GET}   &=  -(\partial_{xx}u_x + \partial_{xy}u_y), \\
  \notag	\partial_{u_x} \det (GET) &= \partial_{u_x} (ac - b^2) \\
  \notag   &= -(\partial_{xx}u_x + \partial_{xy}u_y)c \\ 
  & \quad \qquad + b(\partial_{yx}u_x + \partial_{yy}u_y).
\end{align}
\end{subequations}
In the case $b=0$ the difference to the above derivation is that $\det(GET) = ac$, thus $\partial_{u_x} \alpha$ and $\partial_{u_x} \mathrm{det}(GET)$ should be modified accordingly. The component $\nabla u^{\top} D_{u_y}$ follows the same line of calculations and is therefore omitted here. 

\bibliographystyle{spmpsci}  

\bibliography{JMIV2014}

@INPROCEEDINGS{Blomgren97totalvariation,
    author = {Peter Blomgren and Tony F. Chan and Pep Mulet and C.K. Wong},
    title = {{Total Variation Image Restoration: Numerical Methods and Extensions}},
    booktitle = {in Proc. IEEE ICIP},
    year = {1997},
    pages = {384--387}
}

@ARTICLE{Chen04variableexponent,
    author = {Yunmei Chen and Stacey Levine and Murali Rao},
    title = {{Variable exponent, linear growth functionals in image processing}},
    journal = {SIAM Journal on Applied Mathematics},
    year = {2004},
    volume = {66},
    pages = {1383--1406}
}

@ARTICLE{Bollt2008,
author="Bollt, Erik M. and Chartrand, Rick and Esedo{\u{g}}lu, Selim and Schultz, Pete and Vixie, Kevin R.",
title="{Graduated Adaptive Image Denoising: Local Compromise Between Total Variation and Isotropic Diffusion}",
journal="Advances in Computational Mathematics",
year="2008",
volume="31",
number="1",
pages="61--85",
issn="1572-9044",
}

@Book{aubert-kornprobst:06,
author = {Aubert, G. and Kornprobst, P.},
title = {{Mathematical Problems in Image Processing: Partial Differential Equations and the Calculus of Variations (second edition)}},
year = {2006},
volume = {147},
publisher = {Springer, Berlin},
series = {Applied Mathematical Sciences} 
}

@Inbook{Levin2012,
author="Levin, Anat and Nadler, Boaz and Durand, Fredo and Freeman, William T.",
chapter="Patch Complexity, Finite Pixel Correlations and Optimal Denoising",
title="ECCV 2012",
year="2012",
publisher="Springer Berlin Heidelberg",
address="Berlin, Heidelberg",
pages="73--86",
isbn="978-3-642-33715-4",
}

@article{Oliveira20091683,
title = "Adaptive total variation image deblurring: A majorization–minimization approach ",
journal = "Signal Processing ",
volume = "89",
number = "9",
pages = "1683 - 1693",
year = "2009",
note = "",
issn = "0165-1684",
author = "João P. Oliveira and José M. Bioucas-Dias and Mário A.T. Figueiredo",
}

@INPROCEEDINGS{4409000,
author={Huguet, F. and Devernay, F.},
booktitle={Computer Vision, 2007. ICCV 2007. IEEE 11th International Conference on},
title={A Variational Method for Scene Flow Estimation from Stereo Sequences},
year={2007},
pages={1-7},
ISSN={1550-5499},
month={Oct},}

@ARTICLE{1420392,
author={Bruhn, A. and Weickert, J. and Feddern, C. and Kohlberger, T. and Schn{\"o}rr, C.},
journal={Image Processing, IEEE Transactions on},
title={Variational optical flow computation in real time},
year={2005},
volume={14},
number={5},
pages={608-615},
ISSN={1057-7149},
month={May},}

@book{gilbarg2001elliptic,
  title={Elliptic Partial Differential Equations of Second Order},
  author={Gilbarg, D. and Trudinger, N.S.},
  isbn={9783540411604},
  lccn={00052272},
  series={Classics in Mathematics},
  year={2001},
  publisher={Springer Berlin Heidelberg}
}

@INPROCEEDINGS{5539844,
author={Schmidt, U. and Qi Gao and Roth, S.},
booktitle={Computer Vision and Pattern Recognition (CVPR), 2010 IEEE Conference on},
title={A generative perspective on MRFs in low-level vision},
year={2010},
month={June},
pages={1751-1758},
ISSN={1063-6919},}

@book{DiBenedettoBook,
	Address = {New York},
	Author = {DiBenedetto, Emmanuele},
	Isbn = {0-387-94020-0},
	Mrclass = {35K65 (35-02)},
	Mrnumber = {1230384 (94h:35130)},
	Mrreviewer = {Ya Zhe Chen},
	Pages = {xvi+387},
	Publisher = {Springer-Verlag},
	Title = {Degenerate parabolic equations},
	Year = {1993}}

@book{LadyzenskajaSolonnikovUraltceva,
	Address = {Providence, R.I.},
	Author = {Lady{\v{z}}enskaja, O. A. and Solonnikov, V. A. and Ural'ceva, N. N.},
	Mrclass = {35.62},
	Mrnumber = {0241822 (39 \#3159b)},
	Mrreviewer = {B. Frank Jones, Jr.},
	Pages = {xi+648},
	Publisher = {American Mathematical Society},
	Series = {Translated from the Russian by S. Smith. Translations of Mathematical Monographs, Vol. 23},
	Title = {{Linear and Quasilinear Equations of Parabolic Type}},
	Year = {1968}}

@article{Kass88snakes:active,
author = {Kass, Michael and Witkin, Andrew and Terzopoulos, Demetri},
journal = {International Journal of Computer Vision},
number = {4},
pages = {321--331},
title = {{Snakes: Active contour models}},
volume = {1},
year = {1988}
}

@article{935036,
author = {Ballester, C and Bertalm{\'i}o, M and Caselles, V and Sapiro, G and Verdera, J},
issn = {1057-7149},
journal = {IEEE Transactions on Image Processing},
number = {8},
pages = {1200--1211},
title = {{Filling-In by Joint Interpolation of Vector Fields and Gray Levels}},
volume = {10},
year = {2001}
}

@article{HornShunk1981,
author = {Horn, Berthold K P and Schunck, Brian G},
journal = {Artificial Intelligence},
pages = {185--203},
title = {{Determining Optical Flow}},
volume = {17},
year = {1981}
}

@article{661187,
author = {Chan, T F and Wong, Chiu-Kwong},
issn = {1057-7149},
journal = {IEEE Transactions on Image Processing},
number = {3},
pages = {370--375},
title = {{Total Variation Blind Deconvolution}},
volume = {7},
year = {1998}
}

@article{citeulike:6171513,
    author = {Torralba, A. and Oliva, A.},
    issn = {0954-898X},
    journal = {Network: Computation in Neural Systems},
    month = aug,
    pages = {391--412},
    publisher = {Informa Healthcare},
    title = {{Statistics of natural image categories}},
    year = {2003}
}

@book{Boyd:2004:CO:993483,
 author = {Boyd, Stephen and Vandenberghe, Lieven},
 title = {{Convex Optimization}},
 year = {2004},
 isbn = {0521833787},
 publisher = {Cambridge University Press},
 address = {New York, NY, USA},
}

@article{doi:10.1137/S1064827596299767,
author = {Tony F. Chan and Gene H. Golub and Pep Mulet},
title = {A Nonlinear Primal-Dual Method for Total Variation-Based Image Restoration},
journal = {SIAM Journal on Scientific Computing},
volume = {20},
number = {6},
pages = {1964-1977},
year = {1999}
}

@article{Scherzer2000,
author = {Scherzer, Otmar and Weickert, Joachim},
journal = {Journal of Mathematical Imaging and Vision},
number = {1},
pages = {43--63},
title = {{Relations Between Regularization and Diffusion Filtering}},
volume = {12},
year = {2000}
}

@book{Horn:1985:MA:5509,
address = {New York, NY, USA},
editor = {Horn, Roger A and Johnson, Charles R},
isbn = {0-521-30586-1},
publisher = {Cambridge University Press},
title = {{Matrix Analysis}},
year = {1986}
}

@inproceedings{Lefkimmiatis2013,
author = {Lefkimmiatis, Stamatios and Roussos, Anastasios and Unser, Michael and Maragos, Petros},
booktitle = {SSVM},
editor = {Kuijper, Arjan and Bredies, Kristian and Pock, Thomas and Bischof, Horst},
pages = {48--60},
title = {{Convex Generalizations of Total Variation Based on the Structure Tensor with Applications to Inverse Problems}},
year = {2013}
}

@inproceedings{diva2:274026,
address = {London, Great Britain},
author = {Big\"{u}n, J and Granlund, G\"{o}sta H},
booktitle = {Proceedings of the IEEE First ICCV},
pages = {433--438},
title = {{Optimal Orientation Detection of Linear Symmetry}},
year = {1987}
}

@article{Grasmair210,
author = {Grasmair, Markus and Lenzen, Frank},
journal = {Applied Mathematics \& Optimization},
number = {3},
pages = {323--339},
title = {{Anisotropic Total Variation Filtering}},
volume = {62},
year = {2010}
}

@book{sharma2014digital,
author = {Sharma, G and Bala, R},
isbn = {9781420041484},
publisher = {Taylor \& Francis},
series = {Electrical Engineering \& Applied Signal Processing Series},
title = {{Digital Color Imaging Handbook}},
year = {2014}
}

@incollection{AstroemACCV,
year={2015},
isbn={978-3-319-16630-8},
booktitle={Computer Vision - ACCV 2014 Workshops},
volume={9009},
series={LNCS},
editor={Jawahar, C.V. and Shan, Shiguang},
title={On the Choice of Tensor Estimation for Corner Detection, Optical Flow and Denoising},
publisher={Springer International Publishing},
author={{\AA}str{\"o}m, Freddie and Felsberg, Michael},
pages={16-30},
language={English}
}

@incollection{Astrom2014,
year={2015},
isbn={978-3-319-14611-9},
booktitle={Energy Minimization Methods in Computer Vision and Pattern Recognition (EMMCVPR)},
volume={8932},
series={LNCS},
editor={Tai, Xue-Cheng and Bae, Egil and Chan, TonyF. and Lysaker, Marius},
title={A Tensor Variational Formulation of Gradient Energy Total Variation},
publisher={Springer International Publishing},
author = {{\AA}str{\"o}m, Freddie and Baravdish, George and Felsberg, Michael},
pages={307-320},
language={English}
}

@article{Wang2004,
author = {Wang, Zhou and Bovik, A C and Sheikh, H R and Simoncelli, E P},
issn = {1057-7149},
journal = {IEEE Trans. Image Processing},
month = apr,
number = {4},
pages = {600--612},
title = {{Image quality assessment: from error visibility to structural similarity}},
volume = {13},
year = {2004}
}

@article{Koenderink1984,
author = {Koenderink, Jan J},
journal = {Biological Cybernetics},
number = {50},
pages = {363--370},
title = {{The Structure of Images}},
volume = {370},
year = {1984}
}

@book{lindeberg1993scale,
author = {Lindeberg, T},
isbn = {9780792394181},
publisher = {Springer},
series = {Kluwer international series in engineering and computer science: Robotics: Vision, manipulation and sensors},
title = {{Scale-Space Theory in Computer Vision}},
year = {1993}
}

@inproceedings{SEEDS,
author = {den Bergh, Michael Van and Boix, Xavier and Roig, Gemma and de Capitani, Benjamin and Gool, Luc Van},
booktitle = {ECCV12},
editor = {Fitzgibbon, Andrew W and Lazebnik, Svetlana and Perona, Pietro and Sato, Yoichi and Schmid, Cordelia},
isbn = {978-3-642-33785-7},
pages = {13--26},
publisher = {Springer},
series = {LNCS},
title = {{SEEDS: Superpixels Extracted via Energy-Driven Sampling.}},
volume = {7578},
year = {2012}
}

@article{Lindeberg96_direct_computation,
author = {G{\aa}rding, Jonas and Lindeberg, Tony},
issn = {0920-5691},
journal = {IJCV},
month = feb,
number = {2},
pages = {163--191},
publisher = {Kluwer Academic Publishers},
title = {{Direct computation of shape cues using scale-adapted spatial derivative operators}},
volume = {17},
year = {1996}
}

@inproceedings{Krajsek2010,
address = {San Francisco},
author = {Krajsek, Kai and Scharr, Hanno},
booktitle = {CVPR2010},
isbn = {9781424469857},
issn = {1063-6919},
pages = {2536--2543},
title = {{Diffusion Filtering Without Parameter Tuning : Models and Inference Tools}},
year = {2010}
}

@article{Fberg2011,
author = {Felsberg, M},
issn = {1057-7149},
journal = {Image Processing, IEEE Transactions on},
number = {7},
pages = {1797--1806},
title = {{Autocorrelation-Driven Diffusion Filtering}},
volume = {20},
year = {2011}
}

@book{Weickert96anisotropicdiffusion,
author = {Weickert, Joachim},
publisher = {Teubner-Verlag, Stuttgart, Germany},
title = {{Anisotropic Diffusion in Image Processing}},
year = {1998}
}

@inproceedings{Roussos2010,
author = {Roussos, A and Maragos, P},
booktitle = {ICIP},
isbn = {9781424479948},
pages = {4141--4144},
title = {{Tensor-based image diffusions derived from generalizations of the total variation and beltrami functionals}},
year = {2010}
}

@inproceedings{diva2:269100,
author = {Felsberg, Michael and K\"{o}the, Ullrich},
booktitle = {Scale Space and PDE Methods in Computer Vision},
editor = {Kimmel, R and Sochen, N and Weickert, J},
pages = {192--203},
publisher = {LNCS},
series = {LNCS},
title = {{GET: The connection between monogenic scale-space and Gaussian derivatives}},
volume = {3459},
year = {2005}
}

@article{275632,
author = {Bovik, A C and Maragos, P},
issn = {1053-587X},
journal = {Signal Processing, IEEE Transactions on},
month = feb,
number = {2},
pages = {469--471},
title = {{Conditions for positivity of an energy operator}},
volume = {42},
year = {1994}
}

@incollection{Chan2006,
author = {Chan, T and Esedo{\u{g}}lu, S and Park, F and Yip, A},
booktitle = {Handbook of Mathematical Models in Computer Vision},
chapter = {2},
editor = {Paragios, Nikos and Chen, Yunmei and Faugeras, Olivier},
pages = {17--31},
publisher = {Springer US},
title = {{Total Variation Image Restoration: Overview and Recent Developments}},
year = {2006}
}

@phdthesis{nordberg95,
address = {SE-581 83 Link\"{o}ping, Sweden},
author = {Nordberg, Klas},
school = {Link\"{o}ping University, Sweden},
title = {{Signal Representation and Processing using Operator Groups}},
year = {1995}
}

@inproceedings{1467423,
author = {Buades, A and Coll, B and Morel, J.-M.},
booktitle = {CVPR},
issn = {1063-6919},
pages = {60 -- 65 vol. 2},
title = {{A non-local algorithm for image denoising}},
volume = {2},
year = {2005}
}

@article{Perona90scale-spaceand,
author = {Perona, Pietro and Malik, Jitendra},
issn = {01628828},
journal = {IEEE Transactions, PAMI},
month = jul,
number = {7},
pages = {629--639},
title = {{Scale-space and edge detection using anisotropic diffusion}},
volume = {12},
year = {1990}
}

@article{DBLP:journals/tip/DabovFKE07,
author = {Dabov, Kostadin and Foi, Alessandro and Katkovnik, Vladimir and Egiazarian, Karen O},
journal = {IEEE Transactions on Image Processing},
number = {8},
pages = {2080--2095},
title = {{Image Denoising by Sparse 3-D Transform-Domain Collaborative Filtering}},
volume = {16},
year = {2007}
}

@article{Bredies:2010:TGV:1958729.1958739,
address = {Philadelphia, PA, USA},
author = {Bredies, Kristian and Kunisch, Karl and Pock, Thomas},
issn = {1936-4954},
journal = {SIAM J. Img. Sci.},
month = sep,
number = {3},
pages = {492--526},
publisher = {Society for Industrial and Applied Mathematics},
title = {{Total Generalized Variation}},
volume = {3},
year = {2009}
}

@inproceedings{115702,
author = {Kaiser, J F},
booktitle = {Acoustics, Speech, and Signal Processing, 1990. ICASSP-90., 1990 International Conference on},
issn = {1520-6149},
month = apr,
pages = {381--384 vol.1},
title = {{On a simple algorithm to calculate the `energy' of a signal}},
year = {1990}
}

@incollection{diva2:608779,
author = {{\AA}str\"{o}m, Freddie and Felsberg, Michael and Baravdish, George and Lundstr\"{o}m, Claes},
booktitle = {SSVM2013},
number = {June},
pages = {1--11},
publisher = {Springer Berlin/Heidelberg},
title = {{Targeted Iterative Filtering}},
year = {2013}
}

@article{Kuijper20091023,
title = "Geometrical \{PDEs\} based on second-order derivatives of gauge coordinates in image processing ",
journal = "Image and Vision Computing ",
volume = "27",
number = "8",
pages = "1023 - 1034",
year = "2009",
note = "",
issn = "0262-8856",
author = "Arjan Kuijper",
}

@article{diva2:758277,
author = {Baravdish, George and Svensson, Olof and {\AA}str\"{o}m, Freddie},
title = {{On Backward $p(x)$-Parabolic Equations for Image Enhancement}},
journal = {Numerical Functional Analysis and Optimization},
volume = {36},
number = {2},
pages = {147-168},
year = {2015}
}

@book{books/daglib/0095066,
author = {Granlund, G\"{o}sta H and Knutsson, Hans},
isbn = {978-0-7923-9530-0},
pages = {I--XII, 1--437},
publisher = {Kluwer},
title = {{Signal processing for computer vision.}},
year = {1995}
}

@inproceedings{Felsberg04poidetection,
author = {Felsberg, Michael and Granlund, G\"{o}sta},
booktitle = {In: Pattern Recognition: 26th DAGM Symposium. Volume 3175 of LNCS},
pages = {103--110},
title = {{POI detection using channel clustering and the 2D energy tensor}},
year = {2004}
}

@inproceedings{diva2:543914,
address = {Firenze},
author = {{\AA}str\"{o}m, Freddie and Baravdish, George and Felsberg, Michael},
booktitle = {ECCV12},
pages = {215--228},
publisher = {LNCS, Springer Berlin/Heidelberg},
title = {{On Tensor-Based PDEs and their Corresponding Variational Formulations with Application to Color Image Denoising}},
volume = {7574},
year = {2012}
}

@inproceedings{Forstner87,
author = {F\"{o}rstner, W and G\"{u}lch, E},
booktitle = {ISPRS Intercommission, Workshop, Interlaken, pp. 149-155.},
title = {{A fast operator for detection and precise location of distinct points, corners and centres of circular features}},
year = {1987}
}

@incollection{JWeickert_Disc,
author = {Weickert, Joachim},
booktitle = {Signal processing and pattern recognition},
chapter = {15},
editor = {J\"{a}hne, Bernd and Haussecker, Horst and Beissler, Peter},
pages = {423--451},
publisher = {Academic Press},
series = {Handbook of Computer Vision and Applications},
title = {{Nonlinear Diffusion Filtering}},
year = {1999}
}

@article{Rudin1992,
author = {Rudin, Leonid and Osher, Stanley and Fatemi, Emad},
issn = {0167-2789},
journal = {Physica D},
month = nov,
number = {1-4},
pages = {259--268},
publisher = {Elsevier Science Publishers B. V.},
title = {{Nonlinear total variation based noise removal algorithms}},
volume = {60},
year = {1992}
}

@article{Black1996,
author = {Black, Michael J and Rangarajan, Anand},
journal = {IJCV},
number = {1},
pages = {57--91},
title = {{On the Unification of Line Processes , Outlier Rejection , and Robust Statistics with Applications in Early Vision}},
volume = {19},
year = {1996}
}

@inproceedings{Astrom2013a,
author = {{\AA}str\"{o}m, Freddie and Zografos, Vasileios and Felsberg, Michael},
booktitle = {SCIA},
language = {eng},
pages = {17--20},
title = {{Density Driven Diffusion}},
year = {2013}
}

\end{document}